\documentclass{article}
\usepackage[utf8]{inputenc}
\usepackage[table,xcdraw]{xcolor}
\usepackage{authblk}
\usepackage{graphicx}
\usepackage[a4paper, total={6in, 8in}]{geometry}
\usepackage{amsthm}
\usepackage{amssymb}
\usepackage{multirow}
\usepackage{amsmath}
\usepackage{appendix}
\usepackage{microtype}
\newtheorem{theorem}{Theorem}
\newtheorem{corollary}{Corollary}
\newtheorem{lemma}{Lemma}

\title{A Unified Framework for Generic, Query-Focused, Privacy Preserving and Update Summarization using Submodular Information Measures}

\makeatletter
\newcommand{\printfnsymbol}[1]{%
  \textsuperscript{\@fnsymbol{#1}}%
}
\makeatother

\author[1]{Vishal Kaushal\thanks{equal contribution}}
\author[2]{Suraj Kothawade\printfnsymbol{1}}
\author[1]{Ganesh Ramakrishnan}
\author[3]{Jeff Bilmes}
\author[4]{Himanshu Asnani}
\author[2]{Rishabh Iyer}

\affil[1]{Department of Computer Science, Indian Institute of Technology Bombay}
\affil[2]{Department of Computer Science, University of Texas at Dallas}
\affil[3]{Department of Electrical and Computer Engineering, University of Washington, Seattle}
\affil[4]{School of Technology and Computer Science, Tata Institute of Fundamental Research, Mumbai}

\date{}
\DeclareMathOperator*{\argmax}{arg\,max}
\begin{document}
\maketitle

\begin{abstract}
We study submodular information measures as a rich framework for generic, query-focused, privacy sensitive, and update summarization tasks. While past work generally treats these problems differently ({\em e.g.}, different models are often used for generic and query-focused summarization), the submodular information measures allow us to study each of these problems via a unified approach. We first show that several previous query-focused and update summarization techniques have, unknowingly, used various instantiations of the aforesaid submodular information measures, providing evidence for the benefit and naturalness of these models. We then carefully study and demonstrate the modelling capabilities of the proposed functions in different settings and empirically verify our findings on both a synthetic dataset and an existing real-world image collection dataset (that has been extended by adding concept annotations to each image making it suitable for this task) and will be publicly released.
%These two real-world datasets are built on top of two existing real-world datasets respectively to make them suitable for the problem of query-focused and privacy preserving image collection summarization. 
We employ a max-margin framework to learn a mixture model built using the proposed instantiations of submodular information measures and demonstrate the effectiveness of our approach. While our experiments are in the context of image summarization, our  framework is generic and can be easily extended to other summarization settings (e.g., videos or documents). 
%We also show how some of the past work can be seen as special cases of our formulations. %ROUGE can be seen as a SMI between a set and summary set (same with query-ROUGE as well). This is something we should highlight right in the abstract
\end{abstract}

\section{Introduction}

Recent times have seen unprecedented data growth in different modalities --- e.g., text, images and videos. This has naturally given rise to the need for automatic summarization techniques that automatically summarize data by eliminating redundancy but preserve the main content. As a result, a growing number of publications have been studying applications such as document, image collection, and video summarization. 

There are several flavors of data summarization depending on the domain and user intent. \emph{Generic data summarization} obtains a diverse, yet representative, subset of the entire dataset. While generic summarization does not capture user intent, there are multiple ways of incorporating intent within a summary. One such previously explored approach considers a user query as an additional input to the summarization algorithm. In addition to possessing the desired characteristics of representation, coverage, diversity {\em etc.}, such \emph{query-focused} summaries must additionally be relevant to the input query. Another example of intent is to produce summaries irrelevant to certain topics (for example, one may be uninterested in soccer in a daily sports news summary). An intent could also request a \emph{privacy preserving summarization}, wherein a user may not want elements from or similar to a \emph{private set} in the summary. Summarizing a document while excluding certain confidential or sensitive words, or summarizing an image collection while excluding certain people are examples of privacy-preserving summarization. The key difference between \emph{irrelevance} and \emph{privacy-preserving} summarization is that, while in case of irrelevance one may be willing to occasionally have constraint violations (e.g., occasional appearance of a soccer article), privacy preserving summarization enforces a hard constraint and mandates zero occurence of the private information in the summary. Another form of intent is called \emph{update summarization}, introduced in the Document Understanding Conference (DUC) 2007~\cite{prasad2007iiit}, which selects a summary conditioned on one already seen by the user (thereby additionally requiring diversity relative to the previously seen summary).
While researchers have studied generic, query-focused, and update summarization, the privacy preserving and query irrelevance summarization techniques have received less attention. Furthermore, most previous work deals with these problems each in isolation.\looseness-1

%Further, this notion of intent, as modeled through picking query-relevant items in the summary need not be the only way of modeling user intent. For example, another expression of intent could be a user wanting a summary that doesn't contain certain elements. This calls for a notion of irrelevance. The summary is then supposed to contain elements NOT relevant to the provided query, which in this case can better be called a 'private set'. Such 'privacy preserving summarization'  has obvious applications for example in summarizing a document such that summary should not contain certain confidential or sensitive words or, in summarizing a personal image collection such that the summary doesn't contain images of certain people.  To produce the update summary, some strategies are required to avoid the information which has already been covered by the main summary. 

In this work, we present a powerful framework that can easily express all the aforementioned forms in a unified manner (i.e., it allows us to view generic, query-focused, query-irrelevance, privacy preserving, and update summarization as special cases). Our framework builds upon submodular mutual information measures (to be defined) and employs different submodular functions to parameterize them.\looseness-1

Define $V = \{1, 2, \cdots, n\}$ as a ground-set of $n$ items (e.g., images, video frames, or sentences). Submodular functions~\cite{fujishige2005submodular} are a special class of set functions $f: 2^V \rightarrow \mathbf{R}$ that exhibit a "diminishing returns" property: given subsets $A \subseteq B$ and an item $i \notin B$, submodular functions must satisfy $f(i | A) \triangleq f(A \cup i) - f(A) \geq f(i | B)$. This  makes submodularity suitable for modeling characteristics, such as diversity, coverage, and representation, that are necessary in a good summary. Submodular functions also admit simple and scalable greedy algorithms with constant factor approximation guarantees~\cite{nemhauser1978analysis}. For these reasons, submodular functions have been applied extensively in document, image collection, and video summarization, and have achieved state-of-the-art results~\cite{lin2011class,lin2012learning,tschiatschek2014learning,Gygli2015VideoSB}. Though recent state-of-the-art summarization techniques use deep models for representing \emph{importance} and \emph{relevance}, they are often complemented by submodular functions and DPPs to represent diversity, representation and coverage~\cite{zhang2016video,cho2019improving,tschiatschek2014learning,vasudevan2017query}.
%Submodular functions subsume several important combinatorial constructs such as log-determinants (or determinantal point processes (DPPs)~\cite{kulesza2012determinantal}), facility location problems in operations research, set cover functions, and many others. The submodular summarization framework involves finding an optimal subset that maximizes an appropriate submodular function (or set thereof). From an optimization perspective, a simple greedy algorithm provides an approximation factor of $1 - 1/e$~\cite{nemhauser1978analysis} in the cardinality constrained case, and in practice these greedy solutions are often within 98\% of optimal~\cite{krause2008optimizing}. For these reasons, submodular functions have been applied extensively to the problems of document, image collection, and video summarization, and have achieved state-of-the-art results~\cite{lin2011class,lin2012learning,tschiatschek2014learning,Gygli2015VideoSB}. Though recent state-of-the-art summarization techniques use deep models for representing \emph{importance} and \emph{relevance}, they are often complement submodular functions and DPPs to represent diversity, representation and coverage~\cite{zhang2016video,cho2019improving,tschiatschek2014learning,vasudevan2017query}.\looseness-1 %Rishabh to send link of a document summarization work and also DPP LSTM, deep reinforce etc. fits here. 

Submodular functions alone, however, neither naturally handle relevance/irrelevance to a query nor privacy constraints, both of which are key to the tasks we consider. In this work, we address this via combinatorial information measures defined using submodular functions. In particular we study different instantiations of submodular mutual information and conditional gain functions, their characteristics in query-focused and privacy-preserving settings and demonstrate their effective application in these different flavors of summarization. Our framework applies equally well to single as well as multiple queries and/or private sets.

\section{Related Work and Our Contributions}
% Automatic summarization has drawn a lot of interest and has been % extensively studied in the context of document, images and videos. Here we only give a summary of related work closest to this work.% and defer a lot of other related work to the supplementary material. \textcolor{red}{Is it really necessary to say this?}% in Due to the vast amount of related literature in these areas, we only give a summary of related work here, closest to this paper and defer a lot of other related work to the supplementary material.

We discuss selected related work only that is closest to our own.

\textbf{Image collection summarization: } Approaches for image collection summarization include dictionary learning~\cite{yang2013image}, clustering and k-mediods~\cite{zhao2016visual}, and deep learning~\cite{ozkose2019diverse}. \cite{tschiatschek2014learning} study a class of submodular functions for image collection summarization, and also showed that prior work on image summarization had (unknowingly) used submodular models. \cite{tschiatschek2014learning} also introduce an image summarization corpus that we extend in this work for the query-focused, privacy, and update flavors. %\cite{tschiatschek2014learning} also studies the problem of learning a mixture of appropriate submodular functions for image summarization. A related approach is the use of $k$-DPPs~\cite{Kulesza2011kDPPsFD}. 
%Other approaches include using learning a richer class of parameterized \emph{deep submodular functions}~\cite{dolhansky2016deep}, and learning in the context of noisy functions obtained from data~\cite{singla2015noisy}. 
To the best of our knowledge, query-focused, update, and privacy sensitive summarization have not been studied for image collection summarization. 

\textbf{Summarization in other domains:} Like with image summarization, submodular optimziation has been beneficial in document summarization~\cite{lin2010multi,lin2011class,lin2012learning,li2012multi,chali2017towards,yao2017recent} and video summarization~\cite{zhang2016video,Gygli2015VideoSB,Kaushal2019DemystifyingMV,Kaushal2019AFT}. While most existing work has focused on generic summarization, few papers have studied query based summarization in video and document summarization. 
%Other Combinatorial approaches include  max-sum diversification~\cite{borodin2012max}, dispersion functions (which model diversity)~\cite{dasgupta2013summarization}, and determinantal point processes~\cite{kulesza2012determinantal,Gong2014DiverseSS,mirzasoleiman2017streaming} (which are log-submodular). %Whereas \cite{lin2012learning} learn mixtures of submodular functions for document summarization, \cite{Gygli2015VideoSB,Kaushal2019DemystifyingMV,Kaushal2019AFT} learn submodular functions for video summarization. 
%approach proposed in this paper. 
In the case of video summarization, \cite{vasudevan2017query} study a simple graph-cut based query relevance term (which is a special case of our submodular mutual information framework). Also~\cite{sharghi2016query,sharghi2017query} studies hierarchical DPPs to model both diversity and query relevance, and the query-DPP considered here is a special case of our framework. Others employ deep-learning based methods~\cite{xiao2020convolutional,jiang2019hierarchical}. In terms of document summarization, \cite{lin2011class} define a joint diversity and query relevance term (which we show can also be seen as an instance of our submodular mutual information), and which achieved state of the art results for query-focused document summarization. \cite{li2012multi} use graph-cut functions for query-focused and update summarization (both of which are, again, instances of our submodular information measures).
%~ 

%\noindent \textbf{Submodular Information Measures: } Submodular Information Measures have been studied from a theoretical perspective in~\cite{iyer2020submodular,levin2020online}. Both these papers has studied properties of the submodular information measures, and also related optimization problems. \cite{iyer2020submodular} discusses motivating applications to query and privacy preserving summarization, but mostly considers simple proof of concept experiments. 

%From Jeff: One thing to note, the non-submodularity of mutual informaiton (with entropy of the submdoular function) was known in 2005 (published in 2008, see page 264 of https://jmlr.org/papers/volume9/krause08a/krause08a.pdf) I've fixed the reference since this shows that, in general, I_f(A;Q) is not submodular in A for fixed Q. But I remember discussing it with Andreas in 2005. If looking for more space, I think the LogSet sub-section can go in the supplement.

\subsection{Our Contributions}  We present a unified framework that jointly models notions such as query relevance, irrelevance, privacy sensitivity, and update constraints in summarization, using submodular information measures. We show how our framework generalizes and also unifies past work in this area~\cite{vasudevan2017query,sharghi2016query,lin2011class,lin2010multi,li2012multi,chali2017towards} some of which have unknowingly used special cases of our measures. %While some simple examples have been provided in~\cite{iyer2020submodular}, 
We define the notion of generalized submodular mutual information, and show that the widely-used query-specific~\emph{ROUGE} metric~\cite{lin2004rouge} can be seen as an instance of submodular mutual information, extending the work in~\cite{lin2011class} showing that ROUGE-N is submodular. We then study the modeling and representation power of these models, and show how our framework  handles single and multiple query sets and private sets, and control the hardness of constraints from being `irrelevant' (elements to be avoided as much as possible) to `privacy sensitive' (elements in private set \emph{must not} occur in the summary). We also describe how our framework complements deep learning summarization techniques through the modeling of diversity/representation and query/privacy constraints.  We extend an existing image dataset by annotating every image with concepts making it suitable for the task at hand (and plan to make it publicly available) and, via max-margin learning, we report promising results under all four settings. Although we use image summarization in our experiments, we note that our framework is equally applicable to any summarization domain.

\section{A Unified Summarization Framework}
In this section, we introduce the submodular information measures and show how they offer a unified framework for query, privacy sensitive, irrelevance, update, and generic summarization. We discuss the optimization problem in each setting. This is followed by describing the instantiations of these measures with particular submodular function along with a brief discussion about their representational power. The proofs of all results and lemmas are contained in the Appendix.

Denote $V$ as the ground-set of items to be summarized. We denote by $V^{\prime}$ an auxiliary set that contains user information such as a query or a private set. The auxiliary information provided by the user may not be in the same space as the items $V$ -- for example, if the items in $V$ are images, the query could be text queries. In such a case, we assume we have a \emph{joint} embedding that can represent both the query and the image items, and correspondingly, we can define say similarity matrices between the items in $V$ and $V^{\prime}$. Next, let $\Omega  = V \cup V^{\prime}$ and define a set function $f: 2^{\Omega} \rightarrow \mathbf{R}$. Although $f$ is defined on $\Omega$, summarization is on the items in $V$, i.e., the discrete optimization problem will be only on subsets of $V$.\looseness-1

\subsection{Submodular Information Measures}
Submodular information  measures result from the generalization of information-theoretic quantities like (conditional) entropy and mutual information to general submodular functions. Recall that a function $f$ is submodular iff $f(A) + f(B) \geq f(A \cup B) + f(A \cap B), \forall A, B \subseteq \Omega$. Submodular functions model a number of desirable characteristics, important in summarization problems such as \emph{diversity, relevance, coverage and importance}. Examples of submodular functions include facility location, set cover, log determinants, graph cuts, concave over modular functions etc. Submodular functions have extensively been used in many summarization applications. 

\textbf{Restricted Submodularity:} While submodular functions are expressive, many natural choices are not everywhere submodular. As a result, previous work has considered restricted submodularity~\cite{fujishige2005submodular}. This is particularly relevant to us, since we are considering sets $V$ and $V^{\prime}$, and we are only selecting the summary set from $V$. As a result, we do not require the function to be submodular on all subsets of $\Omega = V \cup V^{\prime}$, and this notion of restricted submodularity will enable us to define generalizations of the submodular mutual information (see below). The idea of restricted submodularity is as follows. Instead of requiring the submodular inequality to hold for all pairs of sets $(A, B) \in 2^{2\Omega}$, we can consider only subsets of this power set. In particular, define a subset $\mathcal C \subseteq 2^{2V}$. Then \emph{restricted submodularity} on $\mathcal C$ satisfies $f(A) + f(B) \geq f(A \cup B) + f(A \cap B), \forall (A, B) \in \mathcal C$. Examples of restricted submodularity include intersecting submodular functions, where $\mathcal C$ consists of sets $A, B$ such that $A \cap B \neq \emptyset$ and crossing submodularity is where $\mathcal C$ consists of sets $(A, B) \in 2^{2\Omega}: A \cap B \neq \emptyset$ and $A \cup B \neq \Omega$. Note that intersecting submodular functions do not satisfy the diminishing returns property, which submodular functions satisfy. %\todo{What does `intersecting and submodular functions do not satisfy...' mean?}

\textbf{Conditional Gain (CG): } We begin by defining the \emph{conditional gain} of a submodular function. Given a set of items $B \subseteq \Omega$, the conditional gain is $$f(A | B) = f(A \cup B) - f(B)$$ Examples of CG include $f(A | P) = f(A \cup P) - f(P), A \subseteq V$ where $P \subseteq V^{\prime}$
is either the \emph{private set} or the \emph{irrelevant set}. %, define $f(A | P) = f(A \cup P) - f(A), A \subseteq V, P \subseteq V^{\prime}$. 
Another example of CG is $f(A | A_0), A, A_0 \in V$ where $A_0$ is a summary chosen by the user \emph{before}. This is important for update summarization where the desired summary should be different from a preexisting one. %The notion of conditional gain function itself is not new, and has been used in many previous works. \textcolor{red}{citations?}

%somewhere we want to mention that by multiplying the image-private similarity with a lambda, we can control the degree of privacy constraint. A \lambda of 1 means privacy aware (summary can still have some points close to private set if the budget permits). Higher the \lambda more is the sensitivity to the private set and lesser is the probability of summary points being close to the points in the private set

\textbf{Submodular Mutual Information (SMI):} We define the \emph{submodular mutual information}~\cite{guillory2011-active-semisupervised-submodular,levin2020online}. Given a set $B \subseteq \Omega$, the submodular mutual information $I_f(A; B)$ is:
%\begin{align}
   $$I_f(A; B) = f(A) + f(B) - f(A \cup B)$$
%\end{align}
The SMI function we will consider in this work is $I_f(A; Q) = f(A) + f(Q) - f(A \cup Q)$, where $Q \subseteq V^{\prime}$ is a query set. Some simple properties of SMI which follow almost immediately from definition is that $I_f(A; B) \geq 0$ and $I_f(A; B)$ is also monotone in $A$ for a fixed $B$. Note that $I_f(A; Q)$ models the mutual coverage, or shared information, between $A$ and $Q$, and is  useful for modeling query relevance in query-focused summarization. Note that $I_f(A; Q)$ is unfortunately not submodular in $A$ for a fixed $Q$ in general~\cite{krause2008near}. However, many of the instantiations we define in this work, actually turn out to be submodular (details below).
 
%We also need to mention the joint formulation - condtional mutual information

\textbf{Generalized Submodular Mutual Information (GSMI): } Now, we define the \emph{generalized submodular mutual information} function $I_f^{\mathcal C}(A; B)$ where $f$ is restricted submodular on $\mathcal C$. Note that $I_f^{\mathcal C}(A; B) \geq 0$ for sets $(A, B) \in \mathcal C$. Given this definition, we consider a $\mathcal C$ defined as follows. Given sets $V$ (the summary space) and $V^{\prime}$ (the auxiliary space), define $\mathcal C(V, C^{\prime}) \subseteq 2^{2\Omega}$ to be such that the sets $(A, B) \in \mathcal C(V, C^{\prime})$ satisfy the following conditions: $A \subseteq V$ or $A \subseteq V^{\prime}$ and $B$ is any set, or $A$ is any set and $B \subseteq V$ or $B \subseteq V^{\prime}$. The following lemma shows that the generalized submodular mutual information on $\mathcal C(V, V^{\prime})$ satisfies non-negativity and monotonicity properties.
\begin{lemma} \label{lemma:genSMI}
Given a restricted submodular function $f$ on $\mathcal C(V, V^{\prime})$, then $I_f(A; A^{\prime}) \geq 0$ for $A \subseteq V, A^{\prime} \subseteq V^{\prime}$. Also, $I_f(A; A^{\prime})$ is monotone in $A \subseteq V$ for fixed $A^{\prime} \subseteq V^{\prime}$ (equivalently, $I_f(A; A^{\prime})$ is monotone in $A^{\prime} \subseteq V^{\prime}$ for fixed $A \subseteq V$).
\end{lemma} %\textcolor{red}{We need to either cite something or talk about proof in supplementary material or something here}
%Mutual information is like joint coverage
%Mutual Info is useless if f is a modular function
\begin{proof}
The non-negativity of the generalized submodular mutual information follows from the definition. In particular, since $A \subseteq V$ and $A^{\prime} \subseteq V^{\prime}$, it holds that $$I_f(A; A^{\prime}) = f(A) + f(A^{\prime}) - f(A \cup A^{\prime}) \geq 0$$ because $f$ is restricted submodular on $C(V, V^{\prime})$. Next, we prove the monotonicity. We have, $$I_f(A \cup j; A^{\prime}) - I_f(A; A^{\prime}) = [f(j | A) - f(j | A \cup A^{\prime})], \forall j \in V \backslash A$$ Given that $f$ is restricted submodular on $\mathcal C(V, V^{\prime})$, it holds that $$[f(j | A) - f(j | A \cup A^{\prime})] \geq 0$$ since $$f(A \cup j) + f(A \cup A^{\prime})\geq f(A) + f(A \cup A^{\prime} \cup j)$$ which follows since the submodularity inequality holds as long as one of the sets is a subset of either $V$ or $V^{\prime}$ (i.e. both subsets do not non-empty intersection with $V$ and $V^{\prime}$. Thus $$I_f(A \cup j; A^{\prime}) - I_f(A; A^{\prime}) \geq 0$$ and hence monotone. 
\end{proof}

\textbf{Conditional Submodular Mutual Information (CSMI): } Given a query set $Q$ and a private set $P$, we would like to select a subset $A \subseteq V$ which has a high similarity with respect to a query set $Q$, while simultaneously being different from the private set $P$. A natural way to do this is by maximizing the conditional submodular mutual information. Given a submodular function $f$, $I_f(A; Q | P)$ is defined as:
%\begin{align}
    $I_f(A; Q | P) = f(A | P) + f(Q | P) - f(A \cup Q | P)$.
%\end{align}
Given a submodular function $f$, the CSMI can either be viewed as the mutual information of the conditional gain function, or the conditional gain of the submodular mutual information.%\textcolor{red}{Should we not say that we give the proof of this in the supplementary material?}
\begin{lemma}\label{lemma:cond-MI}
Given a submodular function $f$ and sets $P, Q$, define $g_P(A) = f(A | P)$ and $h_Q(A) = I_f(A; Q)$. Then $I_f(A; Q | P) = I_{g_P}(A; Q) = h_Q(A | P)$. That is, given a submodular function $f$, the CSMI can either be viewed as the mutual information of the conditional gain function, or the conditional gain of the submodular mutual information.
\end{lemma} 
\begin{proof}
We first prove that $I_f(A; Q | P) = I_{g_P}(A; Q)$. First, recall that:$$I_f(A; Q | P) = f(A \cup P) + f(Q \cup P) - f(A \cup Q \cup P) - f(P)$$
Then, we expand $I_{g_P}(A; Q)$ and observe that:
\begin{align*}
   I_{g_P}(A; Q) &= g_P(A) + g_P(Q) - g_P(A \cup Q)  \\
   &= f(A | P) + f(Q | P) - f(A \cup Q | P)  \\
   &= f(A \cup P) + f(Q \cup P) - f(A \cup Q \cup P) - f(P)  \\
   &= I_f(A; Q | P)
\end{align*}
Next, we show that $I_f(A; Q | P) = h_Q(A | P)$. To  show this, we expand $h_Q(A | P)$:
\begin{align*}
    h_Q(A | P) &= h_Q(A \cup P) - h_Q(P) \\
    &= I_f(A \cup P; Q) - I_f(P; Q) \\
    &= f(A \cup P) + f(Q) - f(A \cup P \cup Q) - f(P) - f(Q) + f(P \cup Q) \\
    &= I_f(A; Q | P)
\end{align*}
Hence proved.
\end{proof}
%\textbf{Submodularity of $I_f$: } We end this subsection, by studying the submodularity of the (conditional) submodular mutual information. We state the result for restricted submodular functions on $\mathcal C(V, V^{\prime})$.
%\begin{lemma}
%Given a restricted submodular function on $\mathcal C(V, C^{\prime})$, $I_f(A; A^{\prime})$ is submodular in $A$ for a fixed set $A^{\prime}$ iff $\forall i \in V$, the gain function $f(i | A)$ is restricted supermodular in $\mathcal C(V, V)$. Similarly, $I_f(A; A^{\prime})$ is submodular in $A^{\prime}$ for a fixed set $A$ iff the gain function $f(i | A^{\prime})$ is restricted supermodular in $\mathcal C(V^{\prime}, V^{\prime})$.
%\end{lemma}
%Since the submodular conditional mutual information is a form of a submodular mutual information (Lemma~\ref{cond-MI}), the submodularity of the conditional submodular MI also follows from the result above.

%Finally, note that $I_f(A; Q | P) = I_f(A; Q)$ if $P = \emptyset$ (i.e. no privacy constraints). Similarly, $I_f(A; Q | P) = f(A | P)$ if $Q = V$ (i.e. we copy the image data points to the query set $V^{\prime}$). Finally, if $Q = V$ and $P = \emptyset$, we get back $I_f(A; Q | P) = f(A)$, i.e. generic summarization.

\subsection{Different Flavors of Summarization}
Given sets $S$ and $T$, and a (restricted) submodular function $f$, consider the following {\it master optimization problem}:
%\begin{align}\label{eqn:master-summ-prob}
    $\max_{A: |A| \leq k}  I_f(A; S | T)$
%\end{align}.
In the sections below, we  discuss how the different types of summarization mentioned above can be seen as special cases of this master problem.

\noindent \textbf{Generic: } Generic summarization is an instance of %equation~\eqref{eqn:master-summ-prob} \textcolor{red}{tofix} 
the master optimization problem with $S = V$ and $T = \emptyset$. Note that we get back the problem of cardinality constrained maximization of $f$, which is generic summarization. 

\noindent \textbf{Query-focused:} With a query set $Q$, this can be written as solving the above %equation~\eqref{eqn:master-summ-prob} \textcolor{red}{tofix} 
with $S = Q$ and $T = \emptyset$.\looseness-1

\noindent \textbf{Privacy Preserving \& Query Irrelevance: } In both cases, we set $S = V$ and $T = P$ (with $P$ the private or irrelevant set) in the master optimization problem. % equation~\eqref{master-summ-prob}\textcolor{red}{tofix}. 
These problems differ in the way $f$ is designed. In particular, for privacy sensitive summarization, we need to ensure no data point similar to the private set is chosen. We  study these  aspects in the next subsection.

\noindent \textbf{Update:} For update summarization, the user has already seen set $A_0$, and we must select $A$ that is diverse w.r.t.\ $A_0$. The master optimization problem solves this %(equation~\eqref{master-summ-prob}\textcolor{red}{tofix}) 
with $T = A_0, S = V$. Also, "query-focused update summarization," where we want a summary similar to $Q$ but different from $A_0$, is achieved by setting $S = Q$ and $T = A_0$. %equation~\eqref{master-summ-prob}\textcolor{red}{tofix}.

\noindent \textbf{Simultaneous Query and Privacy Preserving (Query-Irrelevance) Summarization: } A summary similar to $Q$, yet irrelevant to $P$, is achieved by setting $S = Q$ and $T = P$.

\subsection{Instantiations and Examples}
We instantiate the submodular information measures with different submodular functions and present the expressions in each case along with the derivations. Along the way, we also show that a number of previously studied query-focused and update summarization models are in fact special cases of our framework (for an explicit and separate discussion on this we refer the readers to the Appendix). As introduced above, let $V$ be the ground set of data points which need to be summarized and $V^{\prime}$ be the auxiliary dataset from which the query set $Q$ and the private set $P$ are sampled. That is, $Q \subseteq V^{\prime}$ and $P \subseteq V^{\prime}$. We will denote $A \subseteq V$ as a candidate subset of the data points in $V$. 

\textbf{Set Cover (SC) and Probabilistic Set Cover (PSC):}
We begin with the set cover and probabilistic set cover. Recall the set cover function $f(A) = w(\Gamma(A))$ where for every element $i \in \Omega, \Gamma(i)$ denotes the concepts covered by $i$ and $w$ is a weight vector of the concepts. Here we assume that $V$ and $V^{\prime}$ share the same set of concepts. For example, if our set of instances are images, the concepts could be objects and scenes in the image and we can represent it through a set of concept words. Similarly, if our query is a sentence, it can be represented in the same space of concept words. The probabilistic set cover function is $f(A) = \sum_{i \in U} w_i(1 - P_i(A))$ where $U$ is the set of concepts, and $P_i(A) = \prod_{j \in A} (1 - p_{ij})$, i.e. $P_i(A)$ is the probability that $A$ doesn't cover concept $i$. Intuitively, PSC is a soft version of the SC, which allows for probability of covering concepts, instead of a binary yes/no, as is the case with SC. SMI instantiated with SC (hereafter called SCMI) takes the form $$I_f(A;Q) = w(\Gamma(A) \cap \Gamma(Q))$$ This is intuitive since it measures the joint coverage between the concepts covered by $A$ and $Q$. Likewise, CG with SC (SCCondGain) takes the form, $$f(A | P) = w(\Gamma(A) \setminus \Gamma(P))$$ which again is the difference in the set of concepts covered, and finally, CSMI (SCCondMI) is $$I_f(A;Q|P)=w(\Gamma(A) \cap \Gamma(Q) \setminus \Gamma(P))$$  which essentially represents the mutual coverage of concepts shared with $Q$, while removing the private concepts. For probabilistic set cover, PSCMI is $$I_f(A;Q) = \sum_{i \in U} w_i(1-P_i(A))(1-P_i(Q))$$ Similarly, for CG, PSCCondGain is $$f(A|P)=\sum_{i \in U} w_i(1-P_i(A))P_i(P)$$ and CSMI (PSCCondMI) is $$I_f(A;Q|P)=\sum_{i \in U} w_i (1-P_i(A))(1-P_i(Q))P_i(P)$$ The intuition behind SMI and CG of PSC is also very similar to the SC case, except that the notion of coverage is probabilistic.

\begin{lemma}
The SMI instantiated with SC (hereafter called SCMI) takes the form $I_f(A;Q) = w(\Gamma(A) \cap \Gamma(Q))$. Similarly, CG with SC (SCCG) takes the form, $f(A | P) = w(\Gamma(A) \setminus \Gamma(P))$, while CSMI cn be written as $I_f(A;Q|P)=w(\Gamma(A) \cap \Gamma(Q) \setminus \Gamma(P))$.
\end{lemma}
\begin{proof}
Here $f$ is a set cover function, $f(A) = w(\cup_{a \in A} \Gamma(a))$. Now, $f(A\cup Q) = w(\cup_{c \in A\ \cup Q} \Gamma(c)) $ and we also have $\Gamma(A \cup Q) = \Gamma(A) \cup \Gamma(Q)$ where $\Gamma(A) = \cup_{a \in A} \Gamma(a)$. Thus we have the following result by the inclusion exclusion principle.
\begin{align*}
I_f(A;Q) &= f(A) + f(Q) - f(A \cup Q) \\
 & = w(\Gamma(A)) + w(\Gamma(Q)) - w(\Gamma(A) \cup \Gamma(Q)) \\
 &= w(\Gamma(A) \cap \Gamma(Q))  
\end{align*}
For the Conditional Gain, we have:
\begin{align*}
 f(A|P) &= f(A \cup P) - f(P) \\
  &= w(\Gamma(A) \cup \Gamma(P)) - w(\Gamma(P)) \\ 
  &= w(\Gamma(A) \setminus \Gamma(P))   
\end{align*}
Similarly, we have, for the Conditional Submodular Mutual Information,
\begin{align*}
 f(A;Q|P) &= f(A|P) + f(Q|P) - f(A \cup Q | P) \\
 &= f(A \cup P) - f(P) + f(Q \cup P) - f(P) - f(A \cup Q \cup P) + f(P) \\
 &= f(A \cup P) - f(P) + f(Q \cup P) - f(A \cup Q \cup P)  \\
 &= w(\Gamma(A) \cup \Gamma(P)) - w\Gamma(P) + w(\Gamma(Q) \cup \Gamma(P)) - w(\Gamma(A) \cup \Gamma(Q) \cup \Gamma(P))\\
  &=   [w(\Gamma(A) \cup \Gamma(P)) + w(\Gamma(Q) \cup \Gamma(P)) - w(\Gamma(A) \cup \Gamma(Q) \cup \Gamma(P))] - w\Gamma(P)  \\
  &= w([\Gamma(A) + \Gamma(Q) - \Gamma(A \cup Q)] \cup \Gamma(P)) - w\Gamma(P) \\
  &= w(\Gamma(A) \cap \Gamma(Q) \cup \Gamma(P))  - w\Gamma(P) \\
  &= w(\Gamma(A) \cap \Gamma(Q) \setminus \Gamma(P))
\end{align*}

\end{proof}

\begin{lemma}
The SMI instantiated with PSC (hereafter called PSCMI) takes the form $I_f(A;Q) = \sum_{i \in U} w_i(1-P_i(A))(1-P_i(Q))$. Similarly, CG with PSC (PSCCG) takes the form, $f(A|P)=\sum_{i \in U} w_i(1-P_i(A))P_i(P)$, while the CSMI is $I_f(A;Q|P)=\sum_{i \in U} w_i (1-P_i(A))(1-P_i(Q))P_i(P)$.
\end{lemma}
\begin{proof}
Here $f(A)$ is the probabilistic set cover function: $f(A) = \sum_{i \in U} w_i (1 - \Pi_{a \in A} (1-p_{ia}))$ where $p_{ia}$ is the probability that the element $a \in A$ covers the concept $i$ and $U$ is the set of all concepts. We use $P_i(A) = \Pi_{a \in A} (1-p_{ia})$ which denotes the probability that none of the elements in $A$ cover the concept $i$. Therefore $1 - P_i(A)$ will denote that at least one element in $A$ covers $i$. We have,
\begin{align*}
I_f(A;Q) &= f(A) + f(Q) - f(A \cup Q) \\
&=  \sum_{i \in U} w_i ( 1 - P_i(A) + 1-P_i(Q) - 1 - P_i(A\cup Q)) \\ 
 &=  \sum_{i \in U} w_i (1 - (P_i(A) + P_i(Q) - P_i(A\cup Q))) \\
 \end{align*}
 
 With disjoint $A,Q$ we have: $ P_i(A\cup Q)) = P_i(A) P_i(Q)$ and hence:
 \begin{align*}
     I_f(A;Q) &=  \sum_{i \in U} w_i (1 - (P_i(A) + P_i(Q) - P_i(A) P_i(Q))) \\
     &= \sum_{i \in U} w_i (1 - P_i(A)) (1 - P_i(Q))
 \end{align*}

For the Conditional Gain, we have
\begin{align*}
f(A|P) &= f(A \cup P) - f(P)  \\
&= \sum_{i \in U} w_i [P_i(P) - P_i(A\cup P)] \\ 
&= \sum_{i \in U} w_i [\Pi_{p \in P} (1-p_{ip}) - \Pi_{c \in A \cup P} (1-p_{ic})] \\
&= \sum_{i \in U} w_i [\Pi_{p \in P} (1-p_{ip}) - \Pi_{p \in P} (1-p_{ip}) \Pi_{a' \in A \setminus P} (1-p_{ia'})] \\
&= \sum_{i \in U} w_i [\Pi_{p \in B} (1-p_{ip}) (1 - \Pi_{a' \in A \setminus P} (1-p_{ia'}))] \\
&= \sum_{i \in U} w_i P_i(P)(1 - P_i(A \setminus P))
\end{align*}
Since $A \cap P = \phi$, $A \setminus P=A$, thus, $f(A|P) = \sum_{i \in U} w_i P_i(P)(1 - P_i(A))$

Similarly, for Conditional Submodular Mutual Information using probabilistic Set Cover, we have,
\begin{align*}
f(A;Q|P) &= f(A|P) + f(Q|P) - f(A \cup Q | P) \\
 &= \sum_{i \in U} w_i P_i(P)(1 - P_i(A)) + \sum_{i \in U} w_i P_i(P)(1 - P_i(Q)) - \sum_{i \in U} w_i P_i(P)(1 - P_i(A \cup Q)) \\
 &= \sum_{i \in U}w_i P_i(P)[(1-P_i(A)) + (1-P_i(Q)) - (1-P_i(A)P_i(Q)] \\
 &= \sum_{i \in U}w_i P_i(P)[1-P_i(A)-P_i(Q)+P_i(A)P_i(Q)] \\
 &= \sum_{i \in U}w_i P_i(P)(1-P_i(A))(1-P_i(Q))
\end{align*}

\end{proof}

\textbf{Graph Cut Family (GC):} Another important function is  {\it generalized graph cut}: $f(A) = \sum_{i \in A, j \in V} s_{ij} - \lambda \sum_{i, j \in A} s_{ij}$. The parameter $\lambda$ captures the trade-off between diversity and representativeness. The SMI instantiated with Graph Cut (GCMI) takes the following form: $$I_f(A;Q)=2\lambda \sum_{i \in A} \sum_{j \in Q} s_{ij}$$ It is easy to see that maximizing this function will yield a summary which has a high joint pairwise sum with the query set. In fact, this model has been used in several query-focused summarization works for document summarization~\cite{lin2012submodularity, li2012multi} and video summarization~\cite{vasudevan2017query}, without the authors acknowledging it as form of a submodular mutual information. This re-emphasizes the naturalness of SMI for query-focused summarization. Next, we instantiate the CG with graph cut (GCCondGain). The expression is $$f(A|P)=\sum_{i \in V} \sum_{j \in A} s_{ij} - \lambda \sum_{i, j \in A} s_{ij} - 2 \lambda \sum_{i \in A, j \in P} s_{ij} = f(A) - 2 \lambda \sum_{i \in A, j \in P} s_{ij}$$ Again, the expression is informative. By maximizing $f_{\lambda}(A | P)$, we maximize $f_{\lambda}(A)$ (i.e., the set $A$ should be diverse) while minimizing the SMI between $A$ and $P$ (the sets $A$ and $P$ should be different from each other). We may define a slight generalization of this function, particularly with a goal of being able to handle privacy preserving summarization. As shown in the next section, just maximizing $f_{\lambda}(A | P)$ the way it is currently defined may not penalize private instances enough. A simple solution is to modify the similarity matrix slightly. In particular, the similarity matrix comprises  pairs of elements in $V$ and $V^{\prime}$, and the cross pairs of the elements from $V$ to $V^{\prime}$. We keep the similarity within elements of $V$ and similarity within elements of $V^{\prime}$ unchanged, but we multiply the cross similarity by a $\nu$. The resulting expression, in the case of the graph cut, becomes  %$f_{\lambda, \nu}(A | P) = f_{\lambda}(A) - \nu I_f(A; P)$ \textcolor{red}{Or should we say 
 $$f_{\nu}(A|P)=\sum_{i \in V} \sum_{j \in A} s_{ij} - \lambda \sum_{i, j \in A} s_{ij} - 2 \lambda \sum_{i \in A, j \in P} \nu s_{ij}$$ Because of the way it is constructed, it is still submodular and depending on the value of $\nu$, we can explicitly control the amount of privacy sensitivity. We demonstrate this effect in the next subsection. Finally, we mention that the conditional gain for the graph cut function was used in update summarization~\cite{li2012multi}, further indicating CG is natural for modeling privacy sensitivity, irrelevance, and update summarization. The CGMI expression for GC doesn't make sense mathematically and is not used.\looseness-1  %\textcolor{red}{Should we say that CondMI with GC doesn't make sense?}
 
\begin{lemma}
The SMI instantiated with Graph Cut (GCMI) takes the following form: $I_f(A;Q)=2\lambda \sum_{i \in A} \sum_{j \in Q} s_{ij}$. Similarly, Conditional Gain takes the form $f_{\lambda}(A|P)=\sum_{i \in V} \sum_{j \in A} s_{ij} - \lambda \sum_{i, j \in A} s_{ij} - 2 \lambda \sum_{i \in A, j \in P} s_{ij}$ which is same as $f_{\lambda}(A) - 2 \lambda \sum_{i \in A, j \in P} s_{ij}$. The expression for Conditional Submodular Mutual Information doesn't give a meaningful expression and hence doesn't make sense.
\end{lemma}

\begin{proof}
For generalized graph cut set function, $f(A) = \sum_{i \in V} \sum_{j \in A} s_{ij} - \lambda \sum_{i,j \in A} s_{ij}$. Thus, 
\begin{align*}
I_f(A;Q) &= f(A) + f(Q) - f(A \cup Q) \\ 
&= \sum_{i \in V} \sum_{j \in A} s_{ij} - \lambda \sum_{i,j \in A} s_{ij} 
+ \sum_{i \in V} \sum_{j \in Q} s_{ij} - \lambda \sum_{i,j \in Q} s_{ij}
- \sum_{i \in V} \sum_{j \in A \cup Q} s_{ij} + \lambda \sum_{i,j \in A \cup Q} s_{ij}  \\
\end{align*}
Since $A$ and $Q$ are disjoint, $\sum_{j \in A \cup Q}$ can be broken down as 
$\sum_{j \in A} + \sum_{j \in Q}$ and hence,
\begin{align*}
I_f(A;Q)&= \sum_{i \in V} \sum_{j \in A} s_{ij} - \lambda \sum_{i,j \in A} s_{ij} 
+ \sum_{i \in V} \sum_{j \in Q} s_{ij} - \lambda \sum_{i,j \in Q} s_{ij}
- \sum_{i \in V} \sum_{j \in A} s_{ij} - \sum_{i \in V} \sum_{j \in Q} s_{ij} + \lambda \sum_{i,j \in A \cup Q} s_{ij} \\
&= - \lambda \sum_{i,j \in A} s_{ij} - \lambda \sum_{i,j \in Q} s_{ij} + \lambda \sum_{i,j \in A}s_{ij} + \lambda \sum_{i,j \in Q}s_{ij} + 2\lambda \sum_{i \in A}\sum_{j \in Q}s_{ij} \\
&= 2\lambda \sum_{i \in A}\sum_{j \in Q}s_{ij}
\end{align*}

For the Conditional Gain we derive the expression as follows:

\begin{align*}
    f(A|P) &= f(A \cup P) - f(P) \\
    &= \sum_{i \in V} \sum_{j \in A \cup P} s_{ij} - \lambda \sum_{i,j \in A \cup P} s_{ij} - \sum_{i \in V} \sum_{j \in P} s_{ij} + \lambda \sum_{i,j \in P} s_{ij}  \\
\end{align*}

Since $A \cap P = \phi$ we break down $\sum_{j \in A \cup P}$ alternatively as $\sum_{j \in A - P} + \sum_{j \in P}$ and obtain,

\begin{align*}
    f(A|P) &= \sum_{i \in V} \sum_{j \in A - P} s_{ij} + \sum_{i \in V} \sum_{j \in P} s_{ij} -\lambda \sum_{i,j \in A}s_{ij} - \lambda \sum_{i,j \in P}s_{ij} - 2\lambda \sum_{i \in A}\sum_{j \in P}s_{ij} - \sum_{i \in V} \sum_{j \in P} s_{ij} + \lambda \sum_{i,j \in P} s_{ij} \\
    &= \sum_{i \in V} \sum_{j \in A - P} s_{ij} -\lambda \sum_{i,j \in A}s_{ij} - 2\lambda \sum_{i \in A}\sum_{j \in P}s_{ij} \\
    &= \sum_{i \in V} \sum_{j \in A} s_{ij} -\lambda \sum_{i,j \in A}s_{ij} - 2\lambda \sum_{i \in A}\sum_{j \in P}s_{ij} \\
    &= f(A) - 2\lambda \sum_{i \in A}\sum_{j \in P}s_{ij}
\end{align*}

For deriving the expression for Conditional Submodular Mutual Information we proceed as follows,

Let $$g(A) = f(A|P) = f(A) - 2\lambda \sum_{i \in A}\sum_{j \in P}s_{ij}$$  
Then, 
\begin{align*}
    f(A;Q|P) &= I_g(A;Q) \\
    &= I_f(A;Q) - 2\lambda \sum_{i \in A \cap Q, j \in P}s_{ij}
\end{align*}

Since $A, Q$ are disjoint, the second term is 0 and the first term doesn't have any effect of $P$. Thus, the Conditional Submodular Mutual Information for Graph Cut doesn't make any sense.
\end{proof}

\textbf{Log Determinant (LogDet): } Given a positive semi-definite kernel matrix $L$, denote $L_A$ as the subset of rows and columns indexed by the set $A$. The log-determinant function is $f(A) = \log\det(L_A)$. The log-det function models diversity, and is closely related to a determinantal point process~\cite{kulesza2012determinantal}. Furthermore, denote $S_{AQ}$ as the pairwise similarity matrix between the items in $A$ and $Q$. Then, $$I_f(A; Q) = -\log \det(I - S_A^{-1}S_{AQ}S_Q^{-1}S_{AQ}^T) = \log\det(S_A) - \log\det(S_A - S_{AQ}S_{Q}^{-1}S_{AQ}^T)$$ and we call it LogDetMI. Note that using Neuman series approximation, we can approximate $I_f(A; Q) \approx \log \det(I + S_A^{-1}S_{AQ}S_Q^{-1}S_{AQ}^T)$. Furthermore, if we only consider the data-query similarity, or in other words, define $S_A = I_A, S_Q = I_Q$, we have $I_f(A; Q) \approx \log \det(I + S_{AQ}S_{AQ}^T)$. This term is very similar to the query-focused summarization model used in~\cite{sharghi2016query}. Likewise, the CG function (LogDetCondGain) can be written as $$f(A | P) = \log\det(S_A - S_{AP}S_{P}^{-1}S_{AP}^T)$$ Similar to the graph cut case, we multiply the similarity kernel $S_{AP}$ by $\nu$, to obtain $$f_{\nu}(A | P) = \log\det(S_A - \nu^2 S_{AP}S_{P}^{-1}S_{AP}^T)$$ to control the degree of privacy constraint. Larger values of $\nu$ will encourage privacy sensitive summaries as also demonstrated in subsection below. The CGMI expression for LogDet can be written as $I_f(A; Q | P) = \log \frac{\det(I - S_{P}^{-1} S_{P, Q} S_Q^{-1} S_{P, Q}^T)}{\det(I - S_{A \cup P}^{-1} S_{A \cup P, Q} S_Q^{-1} S_{A \cup P, Q}^T)}$.% \textcolor{red}{What abut LogDeCondMI? Should we say we use Lemma~\ref{lemma:cond-MI} for that?} 

\begin{lemma}
Denote $S_{AQ}$ as the pairwise similarity matrix between the items in $A$ and $Q$. Then, $I_f(A; Q) = -\log \det(I - S_A^{-1}S_{AQ}S_Q^{-1}S_{AQ}^T) = \log\det(S_A) - \log\det(S_A - S_{AQ}S_{Q}^{-1}S_{AQ}^T)$. The Conditional Gain takes the form $f(A | P) = \log\det(S_A - S_{AP}S_{P}^{-1}S_{AP}^T)$ and the Conditional Submodular Mutual Information can be written as $I_f(A; Q | P) = \log \frac{\det(I - S_{P}^{-1} S_{P, Q} S_Q^{-1} S_{P, Q}^T)}{\det(I - S_{A \cup P}^{-1} S_{A \cup P, Q} S_Q^{-1} S_{A \cup P, Q}^T)}$
\end{lemma}
\begin{proof}
Given a positive semi-definite matrix $S$, the Log-Determinant Function is $f(A) = \log \det(S_A)$ where $S_A$ is a sub-matrix comprising of the rows and columns indexed by $A$. The following expressions follow directly from the definitions. The SMI is: $I_f(A; Q) = \log \frac{\det(S_A) \det(S_Q)}{\det(S_{A \cup Q})}$, CG is: $f(A | P) = \log \frac{\det(S_{A \cup P})}{\det(S_P)}$ and CSMI is $I_f(A; Q | P) = \log \frac{\det(S_{A \cup P}) \det(S_{Q \cup P})}{\det(S_{A \cup Q \cup P}) \det(S_P)}$.

Next, note that using the Schur's complement, $\det(S_{A \cup B}) = \det(S_A) \det(S_{A \cup B} \backslash S_A)$ where,
$$S_{A \cup B} \backslash S_A = S_B - S_{AB}^T S_A^{-1} S_{AB}$$
where $S_{AB}$ is a $|A| \times |B|$ matrix and includes the cross similarities between the items in sets $A$ and $B$. Similarly, 
$$S_{A \cup B} \backslash S_B = S_A - S_{AB} S_B^{-1} S_{AB}^T$$
As a result, the Mutual Information becomes:
\begin{align*}
    I_f(A; Q) &= -\log [\det(S_A - S_{AQ} S_Q^{-1} S_{AQ}^T) \det(S_A^{-1})] \\
            &= -\log \det(I - S_{AQ} S_Q^{-1} S_{AQ}^T S_A^{-1}) \\
            &= -\log \det(I - S_A^{-1} S_{AQ} S_Q^{-1} S_{AQ}^T )
\end{align*}
For the CG,  
\begin{align*}
    f(A | P) &= f(A) - I_f(A; P) \\
    &= \log \det(S_A) - \log\det(S_A) + \log \det(S_A - S_{AQ} S_Q^{-1} S_{AQ}^T) \\
    &= \log \det(S_A - S_{AQ} S_Q^{-1} S_{AQ}^T)
\end{align*}

Similarly, the proof of the conditional submodular mutual information follows from the simple observation that:
$$I_f(A; Q | P) = I_f(A \cup P; Q) - I_f(Q; P)$$
Plugging in the expressions of the mutual information of the log-determinant function, from above, we have,
\begin{align*}
    I_f(A \cup P; Q) &= -\log \det(I - S_{A \cup P}^{-1} S_{A \cup PQ} S_Q^{-1} S_{A \cup P Q}^T ) \\
    I_f(Q; P) &= -\log \det(I - S_Q^{-1} S_{QP} S_P^{-1} S_{QP}^T ) \\
    \therefore I_f(A; Q | P) &= \log \det(I - S_{P}^{-1} S_{P, Q} S_Q^{-1} S_{P, Q}^T) - \log \det(I - S_{A \cup P}^{-1} S_{A \cup P, Q} S_Q^{-1} S_{A \cup P, Q}^T) \\
    &= \log \frac{\det(I - S_{P}^{-1} S_{P, Q} S_Q^{-1} S_{P, Q}^T)}{\det(I - S_{A \cup P}^{-1} S_{A \cup P, Q} S_Q^{-1} S_{A \cup P, Q}^T)}
\end{align*}
\end{proof}

\textbf{Facility Location (FL): } Another important function for summarization is  facility location. Given a similarity kernel $S$ and a set $U \subseteq \Omega$, the FL function is $f(A) = \sum_{i \in U} \max_{j \in A} s_{ij}, A \subseteq \Omega$. Depending on the choice of $U$, we have two variants of FL. In the first variant (FL1MI), we set $U$ to be $V$, and this gives us $$I_f(A;Q)=\sum_{i \in V}\min(\max_{j \in A}s_{ij}, \max_{j \in Q}s_{ij})$$ %where $s_{ij}$ measures similarity between element $i$ and element $j$%. 
Analogous to the GCCondGain and LogDetCondGain case, we can up-weigh the query-kernel similarity (say by $\eta$). FL1MI then takes the form: $$I_{f_{\eta}}(A;Q)=\sum_{i \in V}\min(\max_{j \in A}s_{ij}, \eta \max_{j \in Q}s_{ij})$$ Here $\eta$ serves to model a trade-off between query relevance and diversity. We demonstrate the effect of $\eta$ in the next subsection. Next, we consider another variant of FL with $U = \Omega$. %In this case, define $f(A) = \sum_{i \in V^{\prime}} \max_{j \in A \cap V} s_{ij} + \eta \sum_{i \in V} \max_{j \in A \cap V^{\prime}} s_{ij}$. Additionally, 
Additionally we assume that the similarity matrix $S$ is such that $s_{ij} = I(i == j)$, if both $i, j \in V$ or both $i, j \in V^{\prime}$. In other words, we only have the cross similarity between the query and data items. SMI instantiated with this (FL2MI) takes the form $$I_{f_{\eta}}(A; Q) = \sum_{i \in Q} \max_{j \in A} s_{ij} + \eta \sum_{i \in A} \max_{j \in Q} s_{ij}$$ This SMI is also very intuitive. It basically measures the representation of the data points with respect to the query and vice-versa. Moreover, $\eta$ models the trade-off between query similarity and diversity. When $\eta = 0$, it encourages diversity among the query relevant items and when it is very high it encourages sets which are more query-relevant. For Conditional Gain with facility location (FLCondGain), we have, $$f(A|P)= \sum_{i \in V} \max(\max_{j \in A} s_{ij} - \max_{j \in P} s_{ij}, 0)$$ Again, similar to GCCondGain and LogDetCondGain, we can multiply the similarity between the data and private set by $\nu$, to decrease the likelihood of choosing the private elements, giving us $$f_{\nu}(A | P) = \sum_{i \in V} \max(\max_{j \in A} s_{ij} - \nu \max_{j \in P} s_{ij}, 0)$$ Finally, we can also obtain the CGMI expression of $I_f(A; Q | P)$ for FL as $$I_f(A; Q | P) = \sum_{i \in U} \max(\min(\max_{j \in A} s_{ij}, \max_{j \in Q} s_{ij}) - \max_{j \in P} s_{ij}, 0)$$

\begin{theorem}\label{FLMIGen}
Given a similarity kernel $S$, a set $U \subseteq \Omega$ and the facility location (FL) function $f(A) = \sum_{i \in U} \max_{j \in A} s_{ij}, A \subseteq \Omega$ the Mutual Information for FL can be written as $I_f(A;Q)=\sum_{i \in U}\min(\max_{j \in A}s_{ij}, \max_{j \in Q}s_{ij})$.  Similarly, the CG for facility location can be written as $f(A|P)= \sum_{i \in U} \max(\max_{j \in A} s_{ij} - \max_{j \in P} s_{ij}, 0)$ and the expression for Conditional Submodular Mutual Information can be written as: $I_f(A; Q | P) = \sum_{i \in U} \max(\min(\max_{j \in A} s_{ij}, \max_{j \in Q} s_{ij}) - \max_{j \in P} s_{ij}, 0)$.
\end{theorem}
\begin{proof}
Here we have the facility location set function, $f(A) = \sum_{i \in U} \max_{j \in A} s_{ij}$ where $s$ is similarity kernel and $U \subseteq \Omega$. Then,
\begin{align*}
I_f(A;Q) &= f(A) + f(Q) - f(A \cup Q) \\ 
&= \sum_{i \in U} \max_{j \in A} s_{ij} + \max_{j \in Q} s_{ij} - \max_{j \in A \cup Q} s_{ij} \\
&= \sum_{i \in U} \max_{j \in A} s_{ij} + \max_{j \in Q} s_{ij} - \max(\max_{j \in A} s_{ij}, \max_{j \in Q} s_{ij}) \\
&= \sum_{i \in U} \min(\max_{j \in A} s_{ij}, \max_{j \in Q} s_{ij})
\end{align*}

%Assuming $s_{ii} = 1$ is the maximum similarity score in the kernel, for the alternative formulation under the assumption that $U = \Omega$, we can break down the sum over elements in ground set $\Omega$ as follows. For any $i \in A, \max_{j \in A} s_{ij} = 1$ and hence the minimum (over sets $A$ and $Q$) will just be the term corresponding to Q (and a similar argument follows for terms in $Q$). Then we have,

%$$ I_f(A;B) = \sum_{i \in \Omega\setminus (A \cup Q)} \min(\max_{j \in A} s_{ij}, \max_{j \in Q} s_{ij}) 
%    + \sum_{i \in A \setminus Q} \max_{j \in Q} s_{ij} 
%    + \sum_{i \in Q \setminus A} \max_{j \in A} s_{ij}
%    + \sum_{i \in A \cap Q} 1 $$
    
%\textcolor{red}{A special case arises when} $Q = \Omega \setminus A$ as the first and last sums disappear thereby making it look  like a symmetric version of the facility location function: $ I_f(A ; \Omega\setminus A) = \sum_{i \in A} \max_{j \in \Omega \setminus A} s_{ij} + \sum_{i \in \Omega \setminus A} \max_{j \in A} s_{ij}$.

For the Conditional Gain we have
\begin{align*}
    f(A|P) &= \sum_{i \in U} \max(\max_{j \in A} s_{ij}, \max_{j \in P} s_{ij}) -  \max_{j \in P} s_{ij} \\  
 &= \sum_{i \in U} \max(0, \max_{j \in A} s_{ij} - \max_{j \in P} s_{ij})
\end{align*}
Finally, we can get the expression for $I_f(A; Q | P)$ as
\begin{align*}
    I_f(A; Q | P) &= f(A|P) + f(Q|P) - f(A \cup Q | P) \\
    &= \sum_{i in U}[\max(\max_{j in A}s_{ij} - \max_{j in P}s_{ij}, 0) + \max(\max_{j \in Q}s_{ij} - \max_{j \in P}s_{ij}, 0) - \max(\max_{j \in A \cup Q}s_{ij} - \max_{j \in P}s_{ij}, 0)] \\
    &= \sum_{i \in U} \max(\min(\max_{j \in A} s_{ij}, \max_{j \in Q} s_{ij}) - \max_{j \in P} s_{ij}, 0)
\end{align*}
The last step follows from the observation that in $\max(a-c, 0) + \max(b-c, 0) - \max(\max(a,b)-c,0)$ the last term is either the first term or the second term (hence cancelling that out) depending on whether $a > b$ or not.
\end{proof}
Now, we can obtain the expression of FL1MI, FL2MI, FLCG and FLCSMI as special cases.
\begin{corollary}
Setting $U = V$ in the expression of SMI, CG and CSMI in Theorem~\ref{FLMIGen}, we obtain the expression for FL1MI as $I_f(A;Q)=\sum_{i \in V}\min(\max_{j \in A}s_{ij}, \max_{j \in Q}s_{ij})$, FLCG as $f(A|P)= \sum_{i \in V} \max(\max_{j \in A} s_{ij} - \max_{j \in P} s_{ij}, 0)$ and FLCSMI as $I_f(A; Q | P) = \sum_{i \in V} \max(\min(\max_{j \in A} s_{ij}, \max_{j \in Q} s_{ij}) - \max_{j \in P} s_{ij}, 0)$
\end{corollary}
This corollary follows directly from Theorem~\ref{FLMIGen}. 

We can similarly, also obtain the expression for FL2MI.
\begin{corollary}
With $U = \Omega$, and the similarity matrix $S$ is such that $s_{ij} = I(i == j)$, if both $i, j \in V$ or both $i, j \in V^{\prime}$, we obtain the expression for FL2MI as $I_f(A; Q) = \sum_{i \in Q} \max_{j \in A} s_{ij} + \eta \sum_{i \in A} \max_{j \in Q} s_{ij}$
\end{corollary}
\begin{proof}
Assuming $s_{ii} = 1$ is the maximum similarity score in the kernel, for the alternative formulation under the assumption that $U = \Omega$, we can break down the sum over elements in ground set $\Omega$ as follows. For any $i \in A, \max_{j \in A} s_{ij} = 1$ and hence the minimum (over sets $A$ and $Q$) will just be the term corresponding to Q (and a similar argument follows for terms in $Q$). Then we have,

\begin{align}
   I_f(A;B) &= \sum_{i \in \Omega\setminus (A \cup Q)} \min(\max_{j \in A} s_{ij}, \max_{j \in Q} s_{ij}) 
   + \sum_{i \in A \setminus Q} \max_{j \in Q} s_{ij} 
   + \sum_{i \in Q \setminus A} \max_{j \in A} s_{ij}
   + \sum_{i \in A \cap Q} 1  \\
   &= \sum_{i \in V \setminus A} \min(\max_{j \in A} s_{ij}, \max_{j \in Q} s_{ij}) + \sum_{i \in V^{\prime} \setminus Q} \min(\max_{j \in A} s_{ij}, \max_{j \in Q} s_{ij}) + \sum_{i \in A} \max_{j \in Q} s_{ij}
   + \sum_{i \in Q} \max_{j \in A} s_{ij}
\end{align}
This follows because $A \cap Q = \emptyset$. Finally, note that $$\sum_{i \in V \setminus A} \min(\max_{j \in A} s_{ij}, \max_{j \in Q} s_{ij}) + \sum_{i \in V^{\prime} \setminus Q} \min(\max_{j \in A} s_{ij}, \max_{j \in Q} s_{ij}) = 0$$ since $\forall i \in V \setminus A, j \in A, s_{ij} = 0$ and similarly $\forall i \in V^{\prime} \setminus Q, j \in Q, s_{ij} = 0$. This leaves us with $I_f(A; Q) = \sum_{i \in Q} \max_{j \in A} s_{ij} + \eta \sum_{i \in A} \max_{j \in Q} s_{ij}$
\end{proof}

\textbf{ROUGE:} We show that query-specific ROUGE~\cite{lin2004rouge}, a common evaluation metric in document and image summarization~\cite{lin2011class,tschiatschek2014learning}, is an example of GSMI, extending the work in~\cite{lin2011class} showing that ROUGE-N is submodular. Given a summary set $A$ and a reference (or a query) set $Q$, denote $c_i$ as the count of (visual) word $i$. Then $c_i(A)$ and $c_i(Q)$ denote the number of times word $i$ is present in sets $A$ and $Q$. Also, denote by $C$ the set of all words. We can define ROUGE\footnote{Typically, in literature the sum is over the words in the query set, but it is same as this expression because of the $\min$} as:
%\begin{align}
    $$\mbox{ROUGE}_Q(A) = \sum_{i \in C} \min(c_i(A), c_i(Q))$$
%\end{align}.
We also show that $\mbox{ROUGE}_Q(A)$ is a form of a GSMI $I_f(A; Q)$ for a restricted submodular function $f$. In particular, define $f(S) = \sum_{i \in C} \max(c_i(S \cap V), c_i(S \cap V^{\prime}))$ for $S \subseteq \Omega$. Note that $f$ is no longer submodular but it is restricted submodular on $\mathcal C(V, V^{\prime})$. Plugging the expression back into $I_f$, we see that $$I_f(A; Q) = \mbox{ROUGE}_Q(A)$$ (a detailed proof is provided in the Appendix).

\textbf{Concave Over Modular: } We end this section by studying one more important class, sums of concave over modular functions: $$F_{\eta}(A; Q) = \eta \sum_{i \in A} \psi(\sum_{j \in Q}s_{ij}) + \sum_{j \in Q} \psi(\sum_{i \in A} s_{ij})$$ Here $\psi$ is a concave function and again $\eta$ models the trade-off. This class is very general~\cite{bilmes2017deep} and does very well in query-focused extractive document summarization~\cite{lin2011class,lin2012submodularity} with $\delta_1 = 0$ and the concave function as square root. Define $f_{\delta_1, \delta_2}(S) = \delta_1 \sum_{i \in V^{\prime}} \max(\psi(\sum_{j \in S \cap V} s_{ij}), \psi(\sqrt{n}\sum_{j \in S \cap V^{\prime}} s_{ij})) + \delta_2 \sum_{i \in V} \max(\psi(\sum_{j \in S \cap V^{\prime}} s_{ij}), \psi(\sqrt{n}\sum_{j \in S \cap V} s_{ij}))$. Again, $f_{\delta_1, \delta_2}(S)$ is a restricted submodular function on $\mathcal C(V, V^{\prime})$. We can then show (detailed proof in the Appendix), that GSMI with $f_{\delta_1, \delta_2}(S)$ exactly achieves $F_{\delta_1, \delta_2}(A; Q)$.

%\noindent \textbf{Other Functions: } In the supplementary material, 5we go over a few other functions such as saturated coverage and feature based functions, which we skip here in the interest of space. \textcolor{red}{Is it necessary? Paper will look stronger without this. Otherwise it create an effect that this is extra and if this is extra then whatever is presented above is also extra}

%inter query interaction is modeled to some extent by logdet. For others, it is left upto the embedding and the similarity function. Query set can contain many query points and each query point is an image or a set of concepts.

% How some of the earlier techniques can be seen as a special instance of our framework
% Show that many of the past work on query/privacy preserving summarization can be seen as a special cases of our framework
% Gygli's query-focused  - Graph Cut
% Sequential HDPP - Log Det (Query term there)
% Document Summarization (Lin and Bilmes)
% Document Summarization (https://link.springer.com/content/pdf/10.1007/s10489-012-0336-1.pdf)

\subsection{The Representational Power}

\subsubsection{Synthetic Dataset}
We study the behavior of these different instantiations and the effect of the different control parameters by applying them to the summarization tasks on a 2-D synthetic dataset where each element corresponds either to a data point (say, an image) or a query. We characterize the data with four clusters and a few outliers. Specifically, we do a greedy maximization of the above functions individually to arrive at an optimal subset (summary) and study the characteristics of summaries produced by different functions by visualizing them. 2-D visualizations of the points give interesting insights about the functions and different parameters involved.

\textbf{Generic Summarization}: Here we use the standard general forms of the GraphCut, FacilityLocation, DisparitySum, DisparityMin and LogDeterminant individually to produce summaries and verify the expected behavior as seen on synthetic data. While DisparitySum, DisparityMin and LogDeterminant prefer diversity over everything else (hence even picking up outliers)(Figure~\ref{fig:generic-1}), Facility Location prefers to choose representative elements first. GraphCut models both representativeness and diversity with $\lambda$ governing the tradeoff (Figure~\ref{fig:generic-2}).

\begin{figure*}[h]
    \includegraphics[width=\textwidth]{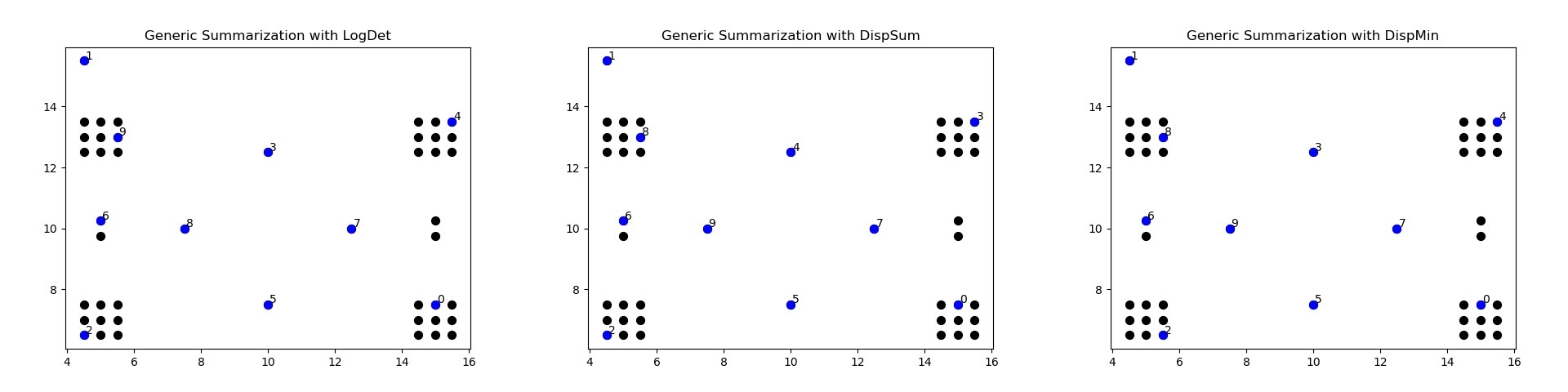}
    \caption{DisparitySum, DisparityMin and LogDeterminant prefer diversity over everything else}
    \label{fig:generic-1}
\end{figure*}

\begin{figure*}[h]
    \includegraphics[width=\textwidth]{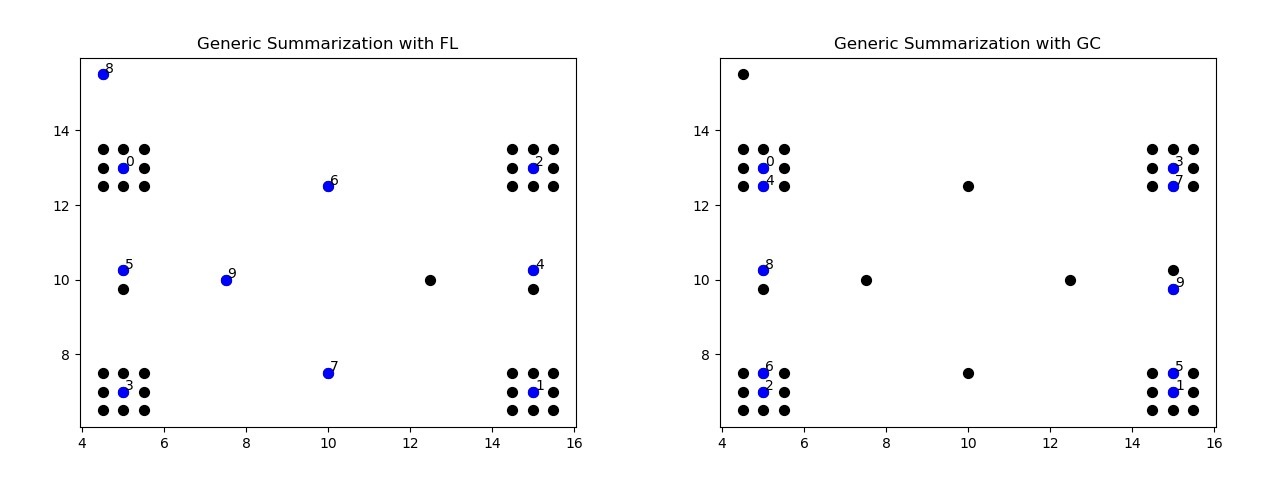}
    \caption{Facility Location prefers to choose representative elements first and then picks the outlier eventually. GraphCut (here with $\lambda$ = 2) doesn't pickup the outlier}
    \label{fig:generic-2}
\end{figure*}

\textbf{Query-focused Summarization}: Since our framework is rich enough to support multiple queries, we carried out extensive analysis to investigate the behavior of the proposed functions in producing summaries for different types of queries, like single queries, multiple queries, outlier queries etc. We report only key representative results here and encourage the readers to refer to similar illustrations and analysis for all other cases in the Appendix. We define \emph{query saturation} as a phenomenon where the function doesn't see any gains in picking more query-relevant items after having picked a few. Similarly, we call a function to be \emph{fair} to the queries if it doesn't starve a query by always picking elements closer to the other query. 

\begin{figure*}[h]
    \includegraphics[width=\textwidth]{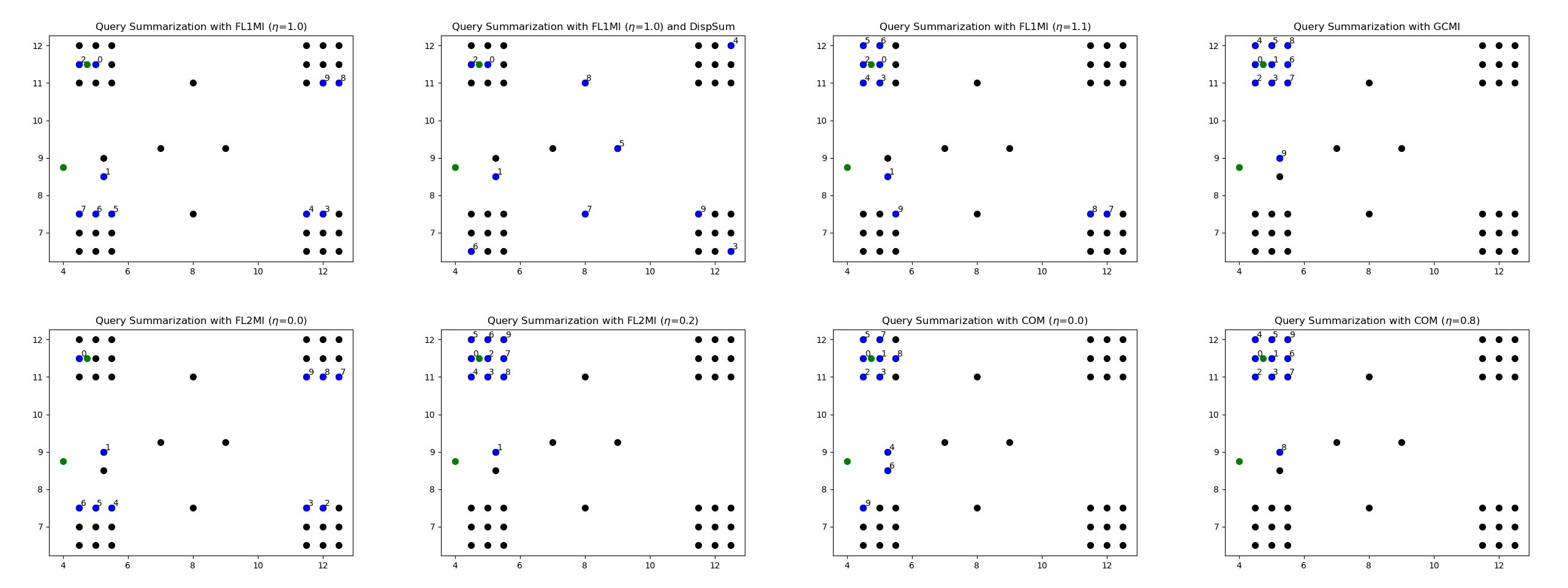}
    \caption{Behavior of different SMI functions and role of the control parameter $\eta$ in producing query focused summaries. Data points are in black, query points are in green and selected summary points are in blue, annotated with the order of their selection. Budget is 10. See text for further details. %Top row, left to right: FL1MI exhibiting fairness and saturation; Even a very small additional diversity to FL1MI gives a good diverse summary; Query relevance of FL1MI increases with $\eta$; GCMI is completely query-relevant, without any saturation and is also unfair to the outlier query. Bottom row, left to right: FL2MI is fair and exhibits saturation; FL2MI immediately switches to query-relevance and also becomes unfair to the outlier query as soon as there is positive non-zero; COM is fair, doesn't exhibit saturation and prefers only query-relevance; As $\eta$ is increased COM reduces on fairness. Data points are in black, query points are in green and selected summary points are in blue, annotated with the order of their selection. Budget is 10.
    }
    \label{fig:query-final}
\end{figure*}

\textbf{Facility Location:} As reported in the first three images of the top row in Figure~\ref{fig:query-final}, summary produced by maximizing FL1MI is fair to both queries and reaches saturation after picking up a few query matching points and then picks up points arbitrarily. Though this arbitrary picking of points is not desirable, saturation offers an advantage in the sense that by combining FLMI with a diversity function, say DisparitySum, with even a very small weight, the points picked up after saturation are diverse. We also observe the effect of $\eta$ in governing the query-relevance and saturation trade-off here. As $\eta$ is increased summary produced by FL1MI increasingly becomes more query-relevant. This is in contrast with the behavior of the alternative formulation FL2MI which in the similar way starts off being fair (with saturation) (first two images of the bottom row in Figure ~\ref{fig:query-final}) but even a slight positive non-zero $\eta$ immediately makes the summary highly query relevant and unfair. Also, FL2MI exhibits saturation by selecting just one matching point to each query as against FL1MI which selects more query-relevant points before reaching saturation. This effect of FL2MI is due to the max in the query-sum term of FL2MI as defined above. 

\textbf{Concave Over Modular:} The behavior of COM is reported in the last two images of the bottom row in Figure~\ref{fig:query-final}. As is evident, unlike the FL variants, COM doesn't exhibit saturation and rather prefers query-relevance. It is also fair to both queries. However, the fairness reduces as the $\eta$ is increased. 

\textbf{Graph Cut:} Summary produced by maximizing GCMI (last image of the top row in Figure~\ref{fig:query-final}) is purely query-relevant, without any saturation and is also, unfair to the outlier query. 

\noindent We defer the comparison and behaviour of log-determinant to the Appendix.
%\begin{figure}
%    \centering
%    \includegraphics[width=0.5\textwidth]{imag%es/query-gc2.png}
%    \caption{GCMI is completely query-relevant, without any saturation and is also unfair to the outlier query}
%    \label{fig:query-gc}
%\end{figure}

\begin{figure}[h]
\begin{center}
    \includegraphics[width=\textwidth]{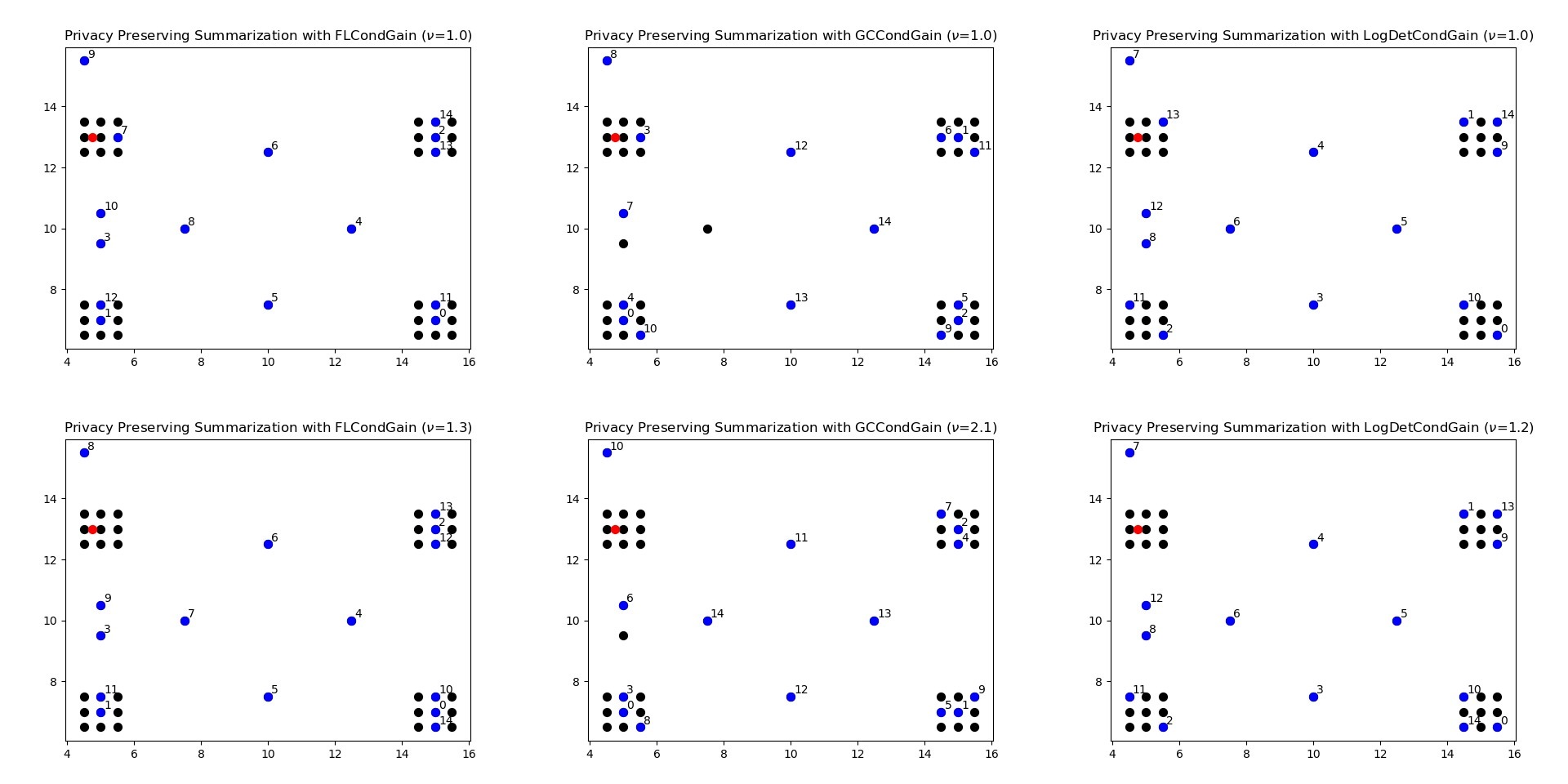}
    \caption{Behavior of different CG functions and role of the control parameter $\nu$ in producing privacy-preserving summaries. See text for more details.
    %Effect of $\nu$ in going from query-irrelevance to privacy-preserving demonstrated in case of FacilityLocation, GraphCut and LogDeterminant instantiations of CondGain. As $\nu$ is increased, the behavior shifts from query-irrelevance to privacy-preserving. Data points are in black, irrelevant/private points are in red and selected summary points are in blue, annotated with the order of their selection. Budget is 15.
    }
    \label{fig:privacy-all}
\end{center}
\end{figure}

\textbf{Privacy Preserving / Query Irrelevance Summarization}: We analyze the summaries produced by GCCondGain, FLCondGain and LogDetCondGain corresponding to GraphCut, FacilityLocation and LogDeterminant respectively and verify that they tend to produce summaries which avoids the query but end up picking it eventually depending on the budget as can be seen in the top row of the figure ~\ref{fig:privacy-all}. As noted earlier, this is expected as there is nothing in their original formulation which explicitly avoids picking the query (so as to become privacy preserving). However, as again already discussed, if such a behavior is desired (privacy preserving, as against query irrelevance), then it can be achieved via the parameter $\nu$ we introduced in their respective formulations. We verify this behavior and observe that increasing $\nu$ produces summary which explicitly avoids the privacy (query) point (bottom row in figure ~\ref{fig:privacy-all}).

\textbf{Unified query-focused and Privacy Preserving Summarization}: We maximize FLCondMI and LogDetCondMI which are the instantiations of CondGain by FacilityLocation and LogDeterminant respectively and visualize the summaries produced to verify that the summary produced has characteristics from both query-focused summarization as well as privacy preserving (or query irrelevance, as desired) summarization. That is, the optimal summary produced is relevant to the query while avoiding the point in the private set, subject to $\nu$. %In the interest of space here we report those results in the Appendix and encourage the readers to refer the illustrations therein. 

\subsubsection{Real Image Data}

Similar to our experiments on synthetic data, we produced image summaries by optimizing different functions individually with varying parameters to understand the behavior of different functions and the role of parameters as seen on real image collection. As expected, we made similar observations as in the case of synthetic data. For example, in the top row of Figure~\ref{fig:real-fl2} we see how the summary produced by FL2MI with $\eta=1$ for query "aircraft, sky" is saturated by picking just one query-relevant item after which it randomly picks items to fill the budget (5). On the other hand, with even a small value of $\eta=0.2$ the behavior changes to become completely query-relevant, picking all 5 elements relevant to the query (Figure~\ref{fig:real-fl2}). On the other hand, as another example, when the summary is produced by maximizing graph-cut based submodular mutual information function it picks all query-relevant items (Figure ~\ref{fig:real-gc})

\begin{figure*}[h]
    \includegraphics[width=\textwidth]{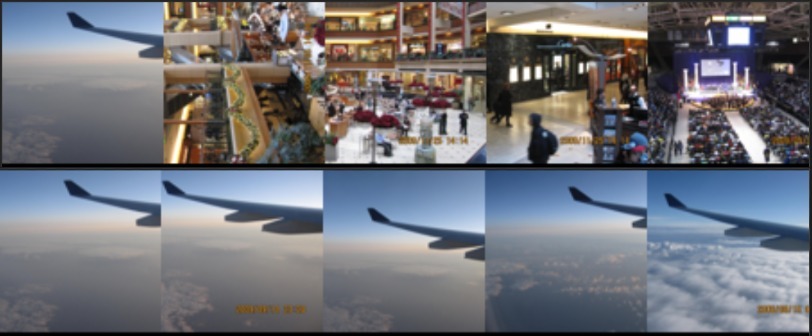}
    \caption{Behavior of FL2MI without $\eta$ (Top) and with $eta$ (Bottom). Summary of budget 5 for image collection number 3 and query "aircraft, sky"}
    \label{fig:real-fl2}
\end{figure*}

\begin{figure*}[h]
    \includegraphics[width=\textwidth]{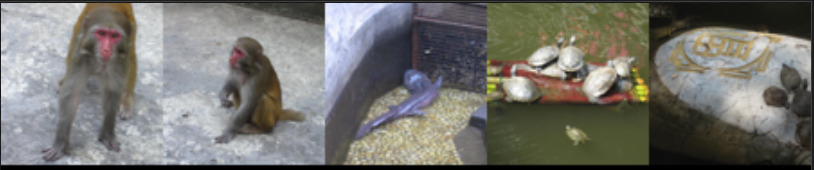}
    \caption{GC (here with $\lambda=1$) produces completely query-relevant summary}
    \label{fig:real-gc}
\end{figure*}

% FL1 = FLC
% GC = FL2 - honors query relevance more than diversity
% LogDet honors diversity more than GC but less than FL1
%FL1 is more diverse than FL2
%Alternate formulations using FL and the trade off between query relevance and diversity controlled by lambda
%FL2 with lambda 0 ~ FL1 

% privacy aware vs privacy preserving
% three alternate formulations to add explcit objective / constraint to ensure distance from private set

\section{Learning Submodular Information Measures}
The different functions presented above produce summaries with different characteristics. Hence, a mixture model would do well to learn internal parameters of these submodular information functions using the human reference summaries to arrive at good summaries. We consider the weights for mixtures of these models, and the trade-off parameters (like $\lambda$, $\eta$, $\nu$) of various functions as our model parameters. We build on prior work that learns mixtures of submodular functions in applications such as document summarization~\cite{lin2012learning}, video summarization~\cite{Gygli2015VideoSB,Kaushal2019DemystifyingMV,Kaushal2019AFT} and image summarization~\cite{tschiatschek2014learning}. We denote our parameter vector as $\Theta = (w, \eta, \lambda, \nu)$, and our mixture model as $F(\Theta)$. Then, given $N$ training examples $(V^{(n)}, Y^{(n)})$ we learn the parameters by optimizing the following large-margin formulation: $\min\limits_{\Theta \geq 0} \frac{1}{N} \sum_{n=1}^{N} L_n(\Theta) + \frac{\lambda}{2}||\Theta||^2$, where $L_n(\Theta)$ is the generalized hinge loss of training example $n$: $L_n(\Theta) = \max\limits_{Y \subset V^{(n)}, |Y| \leq k} (F(Y, x^{(n)}, \Theta) + l_n(Y)) - F(Y^{(n)}, x^{(n)}, \Theta))$. Here $n=1 \dots N$, $Y^{(n)}$ is a human summary for the $n^{th}$ ground set (image collection) $V^{(n)}$ with features $x^{(n)}$.This objective is chosen so that each human reference summary scores higher than any other summary by some margin $l_n(y)$. The parameters $\Theta$ are then learnt using gradient descent. For generic summarization, we add the standard submodular functions modeling representation, diversity, coverage etc. while for query-focused summarization, we use the SMI versions of the functions as defined above along with diversity and representation terms. For privacy-preserving / query-irrelevance summarization, we instantiate the CG versions of different functions along with diversity and representation. We also study the effectiveness of joint learning of query-focused and privacy-preserving summarization using CSMI functions.  Once the parameters are learnt, we instantiate the model with the learnt parameters to get automatically generated summaries. Below we present the specific forms of the mixture model and the objective function in the different cases of Generic, Query-Focused, Privacy-Preserving and Joint Summarization and defer the details of gradient computations to the Appendix. 

\subsection{Generic Summarization}
We denote our dataset of $N$ training examples as $(Y^{(n)}, V^{(n)}, x^{(n)})$ where $n=1 \dots N$, $Y^{(n)}$ is a human summary for the $n^{th}$ ground set (image collection) $V^{(n)}$ with features $x^{(n)}$.

We denote our mixture model in case of generic summarization as $$F(Y, x^{(n)}, w, \lambda) = \sum_{i=1}^M w_i f_i(Y, x^{(n)}, \lambda_i) $$ where $f_1 \dots f_M$ are the instantiations of different submodular functions, $w_i$ their weights and $\lambda_1 \dots \lambda_M$ are their internal parameters respectively, for example the $\lambda$ in case of Graph Cut function defined above.

So the parameters vector in case of generic summarization becomes $\Theta = (w_1 \dots w_m, \lambda_1 \dots \lambda_M)$

Then, $$L_n(\Theta) = \max_{Y \subset V^{(n)}, |Y| \leq k} [\sum_{i=1}^M w_i f_i(Y, x^{(n)}, \lambda_i) + l_n(Y)] - \sum_{i=1}^M w_i f_i(Y^{(n)}, x^{(n)}, \lambda_i) $$

\subsection{Query-Focused Summarization}

We denote our dataset of $N$ training examples as $(Y^{(n)}, V^{(n)}, x^{(n)}, Q^{(n)})$ where $n=1 \dots N$, $Y^{(n)}$ is a human query-summary for the query $Q^{(n)}$ on the $n^{th}$ ground set (image collection) $V^{(n)}$ with features $x^{(n)}$.

We denote our mixture model in case of query summarization as $$F(Y, Q^{(n)}, x^{(n)}, w, \lambda, \eta) = \sum_{i=1}^M w_i I_{f_i}(Y, Q^{(n)}, x^{(n)}, \lambda_i, \eta_i) $$ where $f_1 \dots f_M$ are the instantiations of different submodular mutual information functions, $w_i$ their weights, $\lambda_1 \dots \lambda_M$ are their internal parameters respectively and $\eta_1 \dots \eta_M$ are their query-relevance-diversity tradeoff parameters. 

So the parameters vector in case of query-focused summarization becomes $\Theta = (w_1 \dots w_m, \lambda_1 \dots \lambda_M, \eta_1 \dots \eta_M)$

Then, $$L_n(\Theta) = \max_{Y \subset V^{(n)}, |Y| \leq k} [\sum_{i=1}^M w_i I_{f_i}(Y, Q^{(n)}, x^{(n)}, \lambda_i, \eta_i) + l_n(Y)] - \sum_{i=1}^M w_i I_{f_i}(Y^{(n)}, Q^{(n)}, x^{(n)}, \lambda_i, \eta_i) $$

\subsection{Privacy Preserving Summarization}

We denote our dataset of $N$ training examples as $(Y^{(n)}, V^{(n)}, x^{(n)}, P^{(n)})$ where $n=1 \dots N$, $Y^{(n)}$ is a human privacy-summary for the privacy set $P^{(n)}$ on the $n^{th}$ ground set (image collection) $V^{(n)}$ with features $x^{(n)}$.

We denote our mixture model in case of privacy-preserving summarization as $$F(Y, P^{(n)}, x^{(n)}, w, \lambda, \nu) = \sum_{i=1}^M w_i f_i(Y, P^{(n)}, x^{(n)}, \lambda_i, \nu_i) $$ where $f_1 \dots f_M$ are the instantiations of different conditional gain functions, $w_i$ their weights, $\lambda_1 \dots \lambda_M$ are their internal parameters respectively and $\nu_1 \dots \nu_M$ are their privacy-sensitivity parameters. 

So the parameters vector in case of privacy-preserving summarization becomes $\Theta = (w_1 \dots w_m, \lambda_1 \dots \lambda_M, \nu_1 \dots \eta_M)$

Then, $$L_n(\Theta) = \max_{Y \subset V^{(n)}, |Y| \leq k} [\sum_{i=1}^M w_i f_i(Y, P^{(n)}, x^{(n)}, \lambda_i, \nu_i) + l_n(Y)] - \sum_{i=1}^M w_i f_i(Y^{(n)}, P^{(n)}, x^{(n)}, \lambda_i, \nu_i) $$

\subsection{Joint Summarization}

We denote our dataset of $N$ training examples as $(Y^{(n)}, V^{(n)}, x^{(n)}, Q^{(n)}, P^{(n)})$ where $n=1 \dots N$, $Y^{(n)}$ is a human query-privacy-summary for the query set $Q^{(n)}$ and privacy set $P^{(n)}$ on the $n^{th}$ ground set (image collection) $V^{(n)}$ with features $x^{(n)}$.

We denote our mixture model in case of joint query-focused and privacy preserving summarization as $$F(Y, Q^{(n)}, P^{(n)}, x^{(n)}, w, \lambda, \eta, \nu) = \sum_{i=1}^M w_i f_i(Y, Q^{(n)}, P^{(n)}, x^{(n)}, \lambda_i, \eta_i, \nu_i) $$ where $f_1 \dots f_M$ are the instantiations of different conditional submodular mutual information functions, $w_i$ their weights, $\lambda_1 \dots \lambda_M$ are their internal parameters respectively, $\eta_i$ are their query-relevance vs diversity trade-off parameters and $\nu_1 \dots \nu_M$ are their privacy-sensitivity parameters. 

So the parameters vector in case of joint summarization becomes $\Theta = (w_1 \dots w_m, \lambda_1 \dots \lambda_M, \eta_1 \dots \eta_M, \nu_1 \dots \eta_M)$

Then, $$L_n(\Theta) = \max_{Y \subset V^{(n)}, |Y| \leq k} [\sum_{i=1}^M w_i f_i(Y, Q^{(n)}, P^{(n)}, x^{(n)}, \lambda_i, \eta_i, \nu_i) + l_n(Y)] - \sum_{i=1}^M w_i f_i(Y^{(n)}, Q^{(n)}, P^{(n)}, x^{(n)}, \lambda_i, \eta_i, \nu_i) $$

%\subsection{Max-Margin Learning Framework}

%\subsection{Probabilistic Interpretation}

\section{Experiments and Results}

\subsection{Dataset and Evaluation} We demonstrate the utility and effectiveness of the proposed framework by applying it on a sample task of image collection summarization. We use the image collection dataset of~\cite{tschiatschek2014learning}. The dataset has 14 image collections with 100 images each and provides many human summaries per collection. We extend it by creating dense noun concept annotations for every image thereby making it suitable for our task. We start by designing the universe of concepts based on the 600 object classes in OpenImagesv6~\cite{kuznetsova2018open} and 365 scenes in Places365~\cite{zhou2014learning}. We eliminate concepts common to both (for example, \emph{closet}) to get a unified list of 959 concepts. To ease the annotation process we adopt pseudo-labelling followed by human correction. Specifically for every image we get the concept labels from a Yolov3 model pre-trained on OpenImagesv6 (for object concepts) and a ResNet50 model pre-trained on Place365 (for scene concepts). We then ask 5 human annotators to separately and individually correct the automatically generated labels (pseudo-labels). We then find a consensus over the set of concepts for each image to arrive at the final annotation concept vectors for each image. We have developed a Python GUI tool to ease this pseudo-label correction process (see Figure ~\ref{fig:tool}).

\begin{figure*}[h!]
    \includegraphics[width=\textwidth]{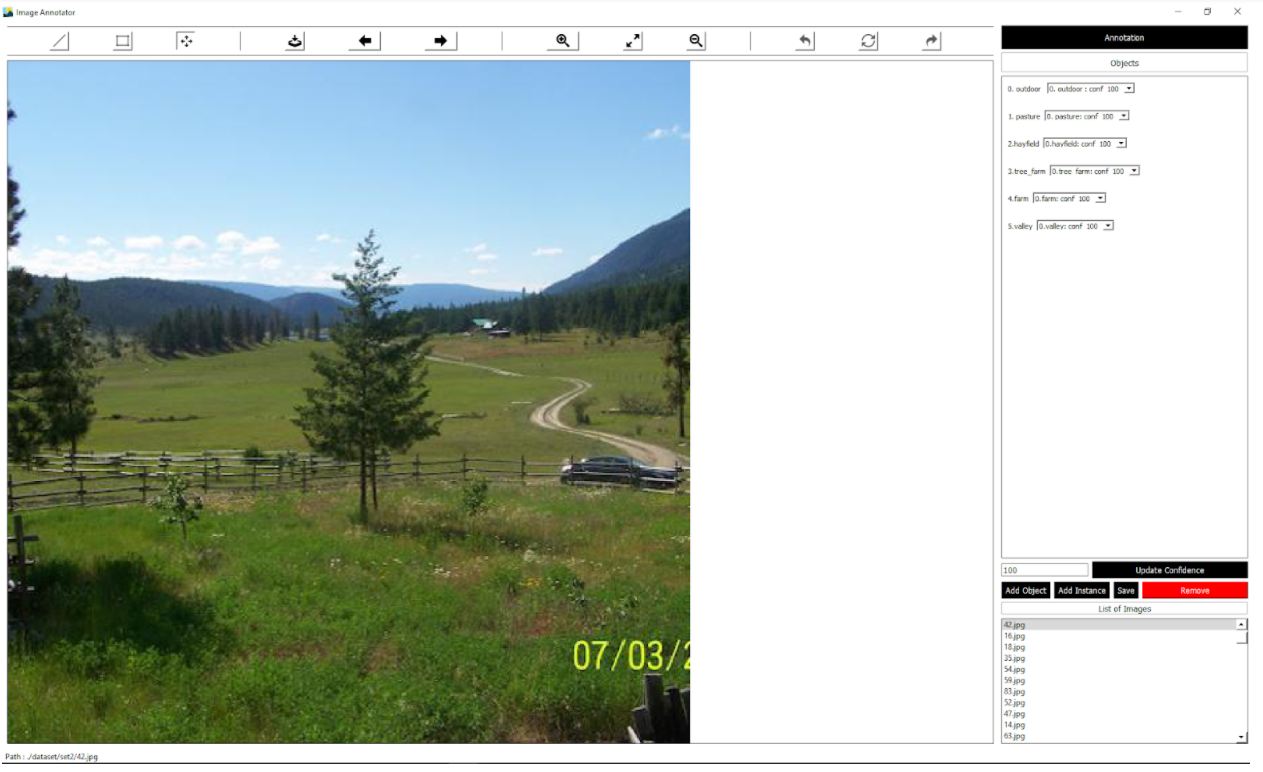}
    \caption{Screenshot of the Annotation Tool}
    \label{fig:tool}
\end{figure*}

In addition to the already available generic human summaries, we augment the dataset with query-focused, privacy-preserving and joint query-focused and privacy-preserving summaries for each image collection. Specifically, we design 2 uni-concept and 2 bi-concept queries / private sets for each image collection to cover different cases like a) both concepts belonging to same image b) both concepts belonging to different images c) only one concept in the image collection. This is similar in spirit to~\cite{sharghi2017query}. We ask a group of 10 human annotators (different from those who annotated for concepts) to create a human summary (of 5 images) for each image collection and query/private pair. To ensure gold standard summaries, we followed this by a verification round. Specifically, we asked at least three annotators to accept/reject the summaries thus produced and deleted those human summaries which were rejected by two or more such verifiers. 

\subsection{Implementation Details} To instantiate the mixture components during training, we represent images using the probabilistic feature vector taken from the output layer of YOLOv3 model~\cite{redmon2018yolov3} pre-trained on OpenImagesv6 and concatenate it with the probability vector of scenes from the output layer of ~\cite{zhou2014learning} trained on Places365 dataset. The queries which are sets of concepts are mapped to a similar feature space as $k$-hot vectors ($k$ being the number of concepts in a query) to facilitate image-query similarity. Thus both images and queries and/or elements in private set are represented using a $L$-dimensional vector where $L$ is the number of concepts in the universe of concepts. While more complex queries and methods for learning joint embedding between text and images could be employed, we chose simpler alternatives to stick to the main focus area of this work. We initialize the parameters using Xavier initialization. We train the mixture model for 20 epochs using 1 - V-ROUGE ~\cite{tschiatschek2014learning} as the margin loss and update parameters using Nesterov's accelerated gradient descent. All results reported here are average V-ROUGE scores across different runs in a leave-one-out setting.

% \begin{figure*}
%     \centering
%     \includegraphics[width = 0.24\textwidth]{images/generic_summ.png}
%     \includegraphics[width = 0.24\textwidth]{images/query-summ.png}
%     \includegraphics[width = 0.24\textwidth]{images/privacy-summ.png}
%     \includegraphics[width = 0.24\textwidth]{images/query-privacy-summ.png}
%     \caption{Results for Real World Image Collection Summarization. Left: Generic Summarization, Second from Left: Query Focused Summarization, Third from Left: Privacy Sensitive, and Fourth: Joint Query and Privacy Preserving Summarization.}
%     \label{fig:image_results}
% \end{figure*}

\subsection{Results and Discussion}

\textbf{Comparison with prior art:} Since there is no explicit past work in query/private image collection summarization, we use the individual components to serve as baselines in two ways - a) to establish the superiority of mixture model over a single component b) as explained above, some individual components can be seen as very similar to some past work and hence serve to offer a rough comparison with those past work. 

We conduct experiments to validate the effectiveness of our framework in case of image-collection summarization as a sample application. We report the results for generic summarization in Table~\ref{tab:generic} and Figure~\ref{fig:generic}. Results for query-focused summarization are reported in Table~\ref{tab:query} and Figure~\ref{fig:query}. Likewise, Tables~\ref{tab:private},~\ref{tab:joint} and Figures~\ref{fig:private}, ~\ref{fig:joint} report the results of privacy-preserving and joint (query-focused and privacy-preserving) summarization respectively.
%In Figure~\ref{fig:image_results} we report\footnote{We place higher resolution charts and other results in the Appendix since we were constrained by space to place smaller images here.}
In particular, we report the average V-ROUGE scores of the learned mixture (red bars) in each of the four summarization settings compared against score of human summaries (green bars), random summaries (orange bars) and baseline scores of individual mixture components (blue bars). We see that our mixture model always outperforms random and all baselines and is second only to human performance. As discussed at length above, different functions model different characteristics and hence fare better than any individual component which can model only one characteristic individually. Due to the highly subjective nature of image collection summarization, human summaries tend to be diverse, with a mix of characteristics and hence a mixture with different components learns the underlying characteristics well. Also, the different control parameters like $\lambda$, $\eta$ and $\nu$ which are critical in producing good summaries under different settings are learnable parameters in our framework along with the weights of the mixture components $\Theta = (w, \eta, \lambda, \nu)$ hence making the model perform better.

%table for generic
\begin{table}[]
    \centering
    \begin{tabular}{|c|c|c|c|c|}
    \hline
    \textbf{Method}& \textbf{Avg} & \textbf{Min} & \textbf{Max} & \textbf{Std Deviation} \\
    \hline
    Mixture & \textbf{0.477} & 0.315 & 0.625 & 0.088 \\
    GC & 0.47 & 0.324 & 0.575 & 0.077 \\
    FL & 0.454 & 0.277 & 0.575 & 0.087 \\
    LogDet & 0.453 & 0.378 & 0.675 & 0.095 \\
    SC & 0.386 & 0.329 & 0.578 & 0.062 \\
    Human & 0.47 & 0.405 & 0.529 & 0.041 \\
    Random & 0.426 & 0.384 & 0.485 & 0.033 \\
    \hline
    \end{tabular}
    \caption{Performance of our mixture model instantiated with different submodular functions for generic image-collection summarization}
    \label{tab:generic}
\end{table}

\begin{figure*}
    \includegraphics[width=\textwidth]{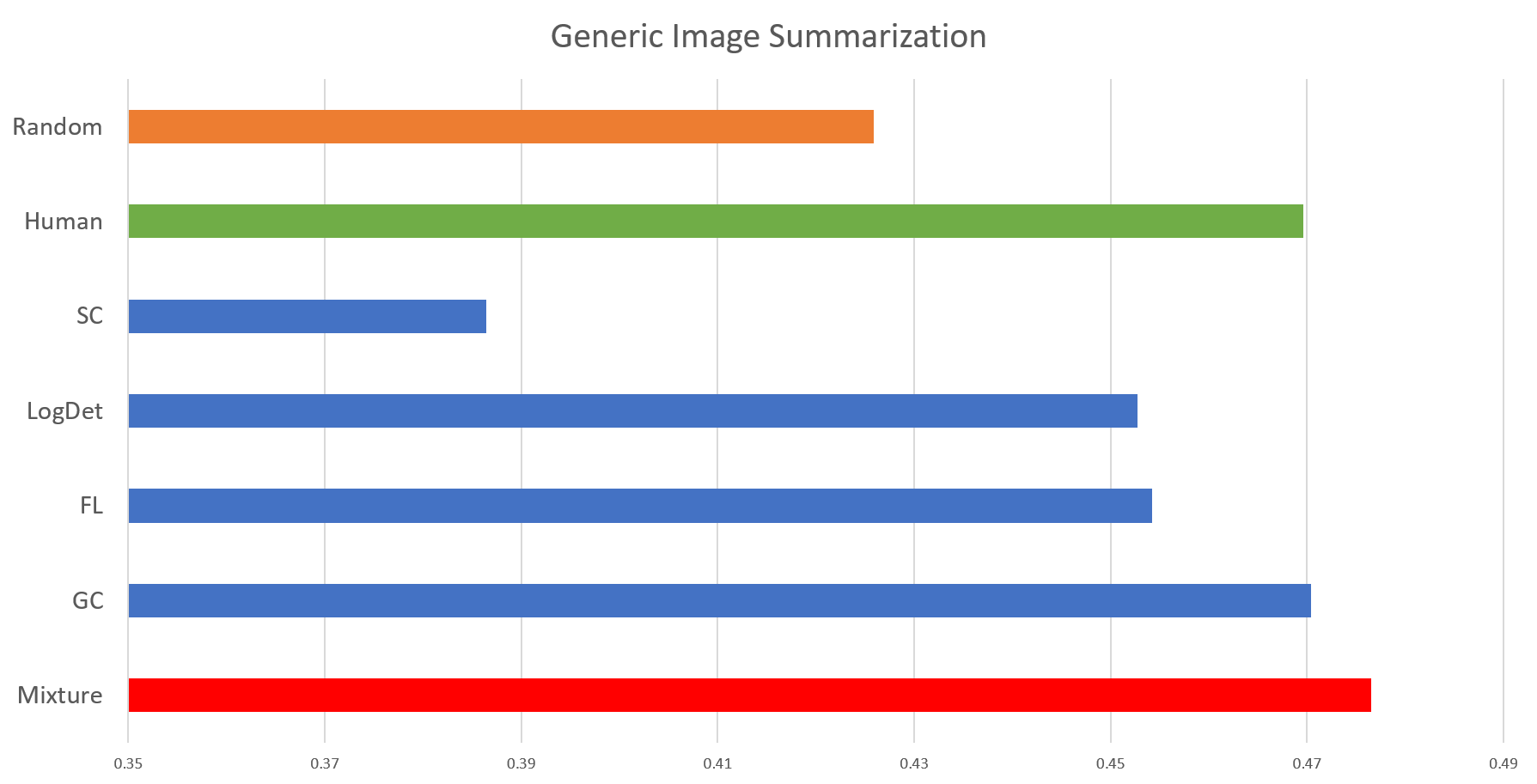}
    \caption{Generic image-collection summarization results}
    \label{fig:generic}
\end{figure*}

%table for query
\begin{table}[]
    \centering
    \begin{tabular}{|c|c|c|c|c|}
    \hline
    \textbf{Method}& \textbf{Avg} & \textbf{Min} & \textbf{Max} & \textbf{Std Deviation} \\
    \hline
    Mixture & \textbf{0.401} & 0.151 & 0.464 & 0.08 \\
    GCMI & 0.318 & 0.204 & 0.616 & 0.115 \\
    FL1MI & 0.318 & 0.168 & 0.441 & 0.085 \\
    FL2MI & 0.318 & 0.204 & 0.616 & 0.115 \\
    LogDetMI & 0.234 & 0.189 & 0.48 & 0.091 \\
    SCMI & 0.273 & 0.136 & 0.438 & 0.073 \\
    PSCMI & 0.262 & 0.112 & 0.482 & 0.09 \\
    COM & 0.318 & 0.204 & 0.616 & 0.115 \\
    Human & 0.543 & 0.114 & 0.468 & 0.095 \\
    Random & 0.292 & 0.432 & 0.726 & 0.077 \\
    \hline
    \end{tabular}
    \caption{Performance of our mixture model instantiated with different submodular mutual information functions for query image-collection summarization}
    \label{tab:query}
\end{table}

\begin{figure*}
    \includegraphics[width=\textwidth]{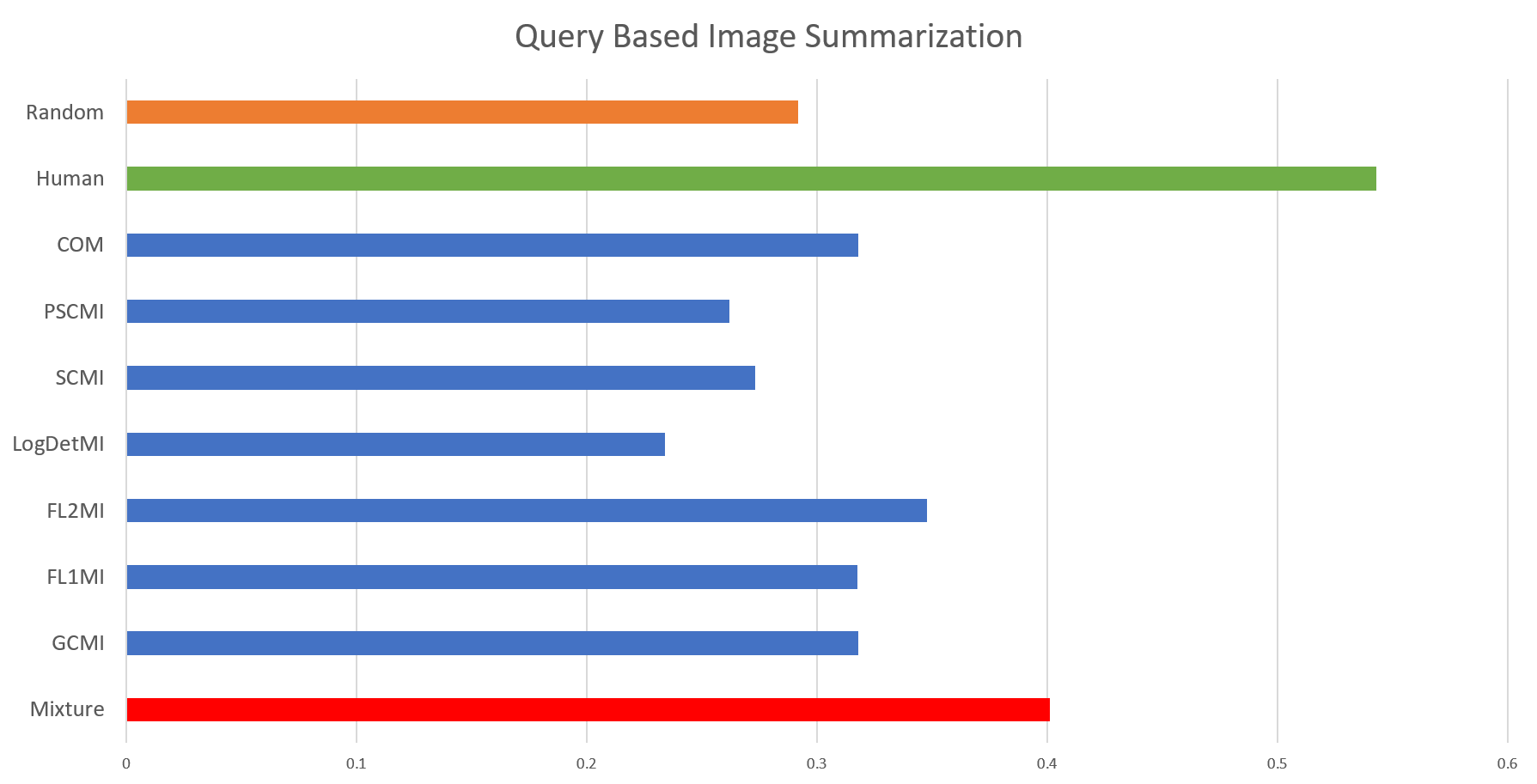}
    \caption{Query-focused image-collection summarization results}
    \label{fig:query}
\end{figure*}

%table for private
\begin{table}[]
    \centering
    \begin{tabular}{|c|c|c|c|c|}
    \hline
    \textbf{Method}& \textbf{Avg} & \textbf{Min} & \textbf{Max} & \textbf{Std Deviation} \\
    \hline
    Mixture & \textbf{0.391} & 0.168 & 0.443 & 0.088 \\
    GCCG & 0.312 & 0.219 & 0.355 & 0.039 \\
    FLCG & 0.274 & 0.168 & 0.42 & 0.072 \\
    SCCG & 0.359 & 0.203 & 0.457 & 0.074 \\
    PSCCG & 0.41 & 0.231 & 0.617 & 0.1 \\
    Human & 0.481 & 0.42 & 0.578 & 0.051 \\
    Random & 0.306 & 0.231 & 0.374 & 0.039 \\
    \hline
    \end{tabular}
    \caption{Performance of our mixture model instantiated with different conditional gain functions for privacy-preserving image-collection summarization}
    \label{tab:private}
\end{table}

\begin{figure*}
    \includegraphics[width=\textwidth]{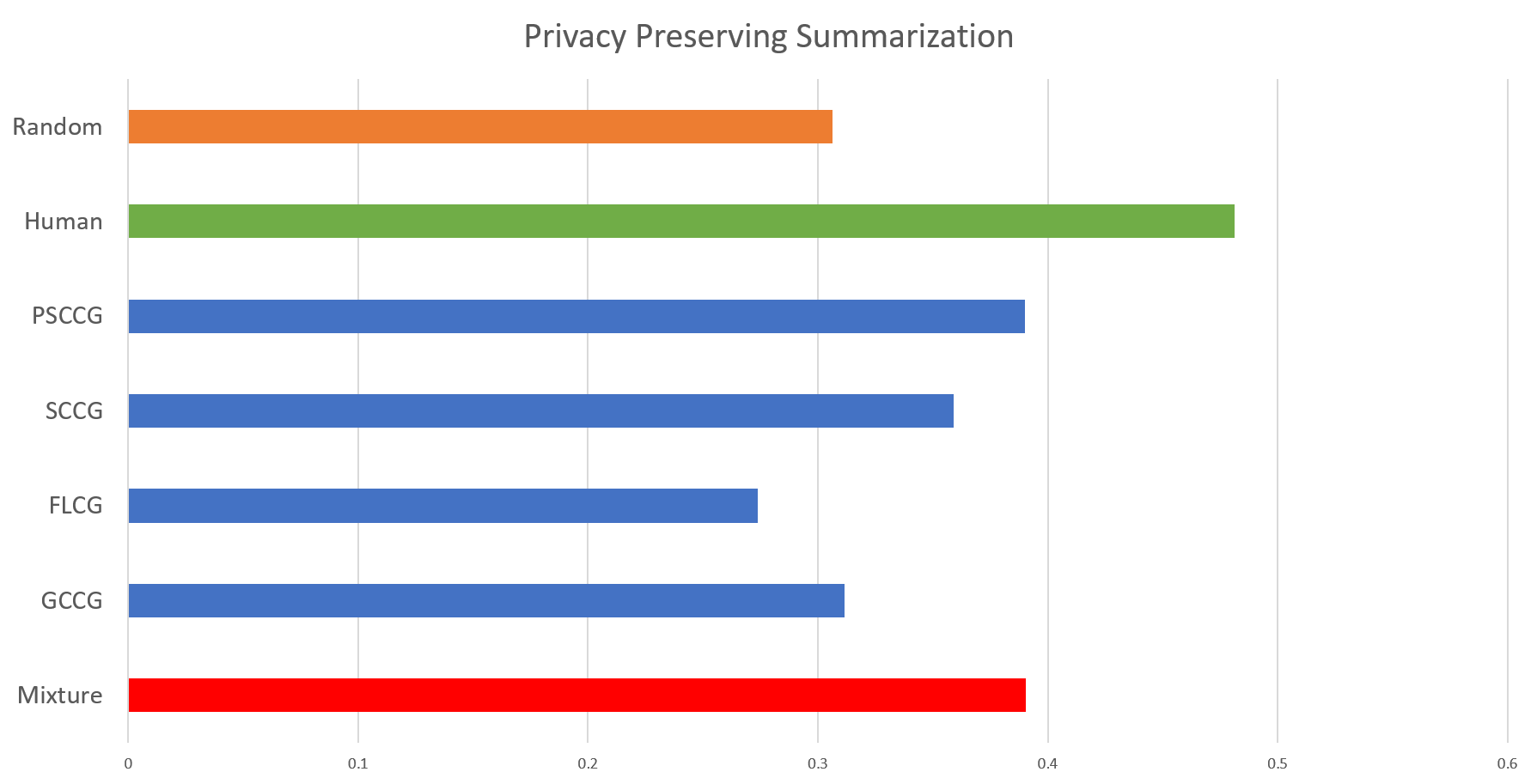}
    \caption{Privacy-preserving image-collection summarization results}
    \label{fig:private}
\end{figure*}

%table for joint
\begin{table}[]
    \centering
    \begin{tabular}{|c|c|c|c|c|}
    \hline
    \textbf{Method}& \textbf{Avg} & \textbf{Min} & \textbf{Max} & \textbf{Std Deviation} \\
    \hline
    Mixture & \textbf{0.314} & 0.103 & 0.505 & 0.104 \\
    SCCondMI & 0.287 & 0.199 & 0.449 & 0.068 \\
    PSCCondMI & 0.29 & 0.201 & 0.383 & 0.057 \\
    FLCondMI & 0.302 & 0.169 & 0.581 & 0.118 \\
    LogDetCondMI & 0.273 & 0.082 & 0.409 & 0.099 \\
    Human & 0.628 & 0.457 & 1 & 0.154 \\
    Random & 0.304 & 0.222 & 0.391 & 0.053 \\
    \hline
    \end{tabular}
    \caption{Performance of our mixture model instantiated with different conditional submodular mutual information functions for joint image-collection summarization}
    \label{tab:joint}
\end{table}

\begin{figure*}
    \includegraphics[width=\textwidth]{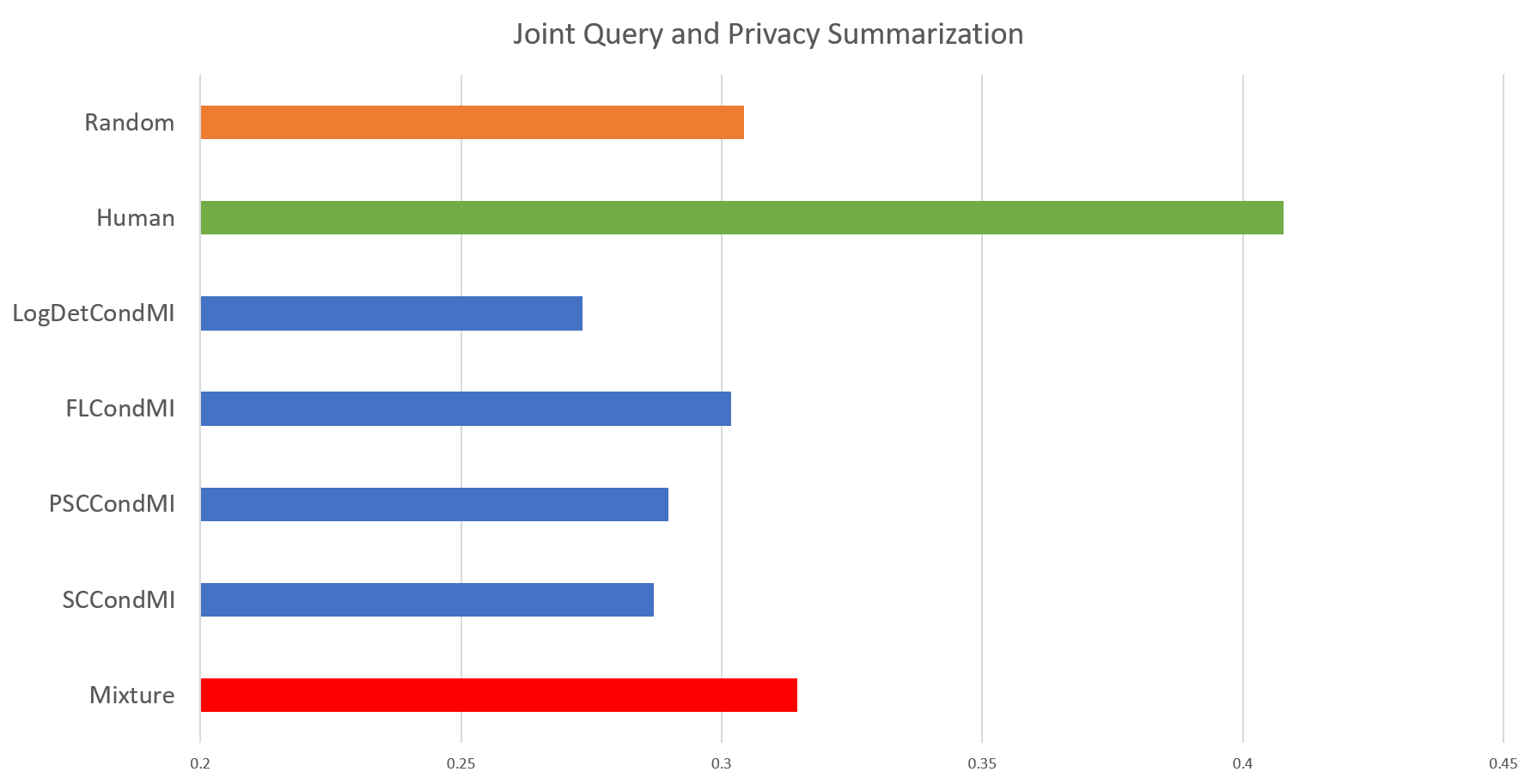}
    \caption{Joint image-collection summarization results}
    \label{fig:joint}
\end{figure*}

\subsection{Ablation Experiments to Study the Effect of $\eta$ and $\nu$}

We performed ablation experiments to quantitatively study the effect of introducing the learnable parameters $\eta$ and $\nu$ on improving the performance of the mixture model. We report the comparison in Table~\ref{tab:ablation}. As expected, the model learnt with these parameters performs better than the one learnt without them. 

\begin{table}[h!]
\centering
\begin{tabular}{|l|l|r|r|r|r|}
\hline
\textbf{Mode}                     & \textbf{Setting}               & \multicolumn{1}{l|}{\textbf{Avg}} & \multicolumn{1}{l|}{\textbf{Min}} & \multicolumn{1}{l|}{\textbf{Max}} & \multicolumn{1}{l|}{\textbf{Std Deviation}} \\
\hline
\multirow{2}{*}{Query}   & With $\eta$ and $\nu$ & \textbf{0.401}                   & 0.151                   & 0.464                   & 0.08                              \\
                         & W/o $\eta$ and $\nu$  & 0.341                   & 0.114                   & 0.491                   & 0.093                             \\ \hline
\multirow{2}{*}{Private} & With $\eta$ and $\nu$ & \textbf{0.391}                   & 0.168                   & 0.443                   & 0.088                             \\
                         & W/o $\eta$ and $\nu$  & 0.302                   & 0.099                   & 0.426                   & 0.08                              \\ \hline
\multirow{2}{*}{Joint}   & With $\eta$ and $\nu$ & \textbf{0.314}                   & 0.103                   & 0.505                   & 0.104                             \\
                         & W/o $\eta$ and $\nu$  & 0.249                   & 0.1                     & 0.524                   & 0.12    \\               \hline
 
\end{tabular}
\caption{Ablation study to compare the model learnt without the additional parameters $\eta$ and $\nu$ with the model learnt with them}
\label{tab:ablation}
\end{table}

\section{Conclusion}
We present a unified approach for generic, query-focused, privacy sensitive, and update summarization tasks
through a rich framework offered by submodular information measures that also helps generalize past work in this area. 
Specifically, we show that a number of previously used approaches for query focused and update summarization, can in fact be viewed as maximizing the SMI and CG functions. We also show that the widely used \emph{ROUGE} metric is an instance of GSMI (generalized SMI). We then carefully study and demonstrate the modelling and representation capabilities of the proposed functions in various settings and empirically verify our findings on both a synthetic data and a real-world, privacy-preserving and query-focused image collection dataset that we intend to release with this paper.  

\section*{Acknowledgements}

This work is supported in part by the Ekal Fellowship (www.ekal.org) and National Center of Excellence in Technology for Internal Security, IIT Bombay (NCETIS, https://rnd.iitb.ac.in/node/101506)

\bibliographystyle{abbrv}
\bibliography{main}

\appendix
\setcounter{lemma}{0}
\section{Some Past Work as Special Cases of Our Framework}

Several query-focused summarization works for document summarization~\cite{lin2012submodularity, li2012multi} and video summarization~\cite{vasudevan2017query} use GCSMI. \cite{vasudevan2017query} study a simple graph-cut based query relevance term (which is a special case of our submodular mutual information framework). Also~\cite{sharghi2016query,sharghi2017query} studies hierarchical DPPs to model both diversity and query relevance, and the query-DPP considered here is a special case of our framework. The query-focused summarization model used in~\cite{sharghi2016query} is very similar to our LogDetSMI Conditional gain for the graph cut function was used in update summarization~\cite{li2012multi}. In terms of document summarization, \cite{lin2011class} define a joint diversity and query relevance term (which we show can also be seen as an instance of our submodular mutual information), and which achieved state of the art results for query-focused document summarization. \cite{li2012multi} use graph-cut  functions for query-focused and update summarization (both of which are, again, instances of our submodular information measures).

\subsection{Query-specific ROUGE is an example of GSMI}
We also show that $\mbox{ROUGE}_Q(A)$ is a form of a GSMI $I_f(A; Q)$ for a restricted submodular function $f$.
In particular, define $f(S) = \sum_{i \in C} \max(c_i(S \cap V), c_i(S \cap V^{\prime}))$ for $S \subseteq \Omega$. Note that $f$ is no longer submodular but it is restricted submodular on $\mathcal C(V, V^{\prime})$. Plugging the expression back into $I_f$, we see that $I_f(A; Q) = \mbox{ROUGE}_Q(A)$.
\begin{lemma}
Given a restricted submodular function $f(S) = \sum_{i \in C} \max(c_i(S \cap V), c_i(S \cap V^{\prime}))$ for $S \subseteq \Omega$, $I_f(A; Q) = \mbox{ROUGE}_Q(A)$.
\end{lemma}
\begin{proof}
We first expand out the expression for $I_f(A; Q)$. Note that $f(S) = \sum_{i \in C} \max(c_i(S \cap V), c_i(S \cap V^{\prime}))$ for $S \subseteq \Omega$ and hence $f(A) = \sum_{i \in C} c_i(A)$ and $f(Q) = \sum_{i \in C} c_i(Q)$. Hence, 
\begin{align}
    I_f(A; Q) &= f(A) + (Q) - f(A \cup Q) \\
    &= \sum_{i \in C} c_i(A) + \sum_{i \in C} c_i(Q) - \sum_{i \in C} \max(c_i(A), c_i(Q)) \\
    &= \sum_{i \in C} \min(c_i(A), c_i(Q)) \\
    &= \mbox{ROUGE}_Q(A)
\end{align}
Finally, note that for any set $A, B \subseteq V$ or $V^{\prime}$, it holds that $f(A) + f(B) \geq f(A \cup B) + f(A \cap B)$. Similarly, for any sets $A \subseteq V$ and $B \subseteq V^{\prime}$, $f(A) + f(B) \geq f(A \cup B)$ and hence the $f$ defined here is restricted submodular on $\mathcal C(V, V^{\prime})$. 
\end{proof}

\subsection{Concave Over Modular}
Define the following query based submodular function:
\begin{align}
  F_{\eta}(A; Q) = \eta \sum_{i \in A} \psi(\sum_{j \in Q}s_{ij}) + \sum_{j \in Q} \psi(\sum_{i \in A} s_{ij})
\end{align}
This is a very general class of functions ~\cite{bilmes2017deep} and does very well in query-focused extractive document summarization~\cite{lin2011class,lin2012submodularity} with $\eta = 0$ and the concave function as square root. Next, define the following restricted submodular function:
\begin{align}
    f_{\eta}(S) = \eta \sum_{i \in V^{\prime}} \max(\psi(\sum_{j \in S \cap V} s_{ij}), \psi(\sqrt{n}\sum_{j \in S \cap V^{\prime}} s_{ij})) +  \sum_{i \in V} \max(\psi(\sum_{j \in S \cap V^{\prime}} s_{ij}), \psi(\sqrt{n}\sum_{j \in S \cap V} s_{ij}))
\end{align}
The following result connects GSMI with $f_{\eta}$ with $F_{\eta}(A; Q)$.
\begin{lemma}
The function $f_{\eta}(S)$ is a restricted submodular function on $\mathcal C(V, V^{\prime})$. Furthermore the GSMI with $f_{\eta}$ is exactly $F_{\eta}(A; Q)$, given a kernel matrix which satisfies $s_{ij} = 1(i == j)$ for $i, j \in V$ or $i, j \in V^{\prime}$.
\end{lemma}
\begin{proof}
Assume that the kernel matrix $s_{ij} \leq 1, \forall i, j \in \Omega$. Also, we are given that $s_{ij} = 1(i == j)$ for $i, j \in V$ or $i, j \in V^{\prime}$. Next, notice that:
\begin{align}
    f(A) = \eta \sum_{i \in V^{\prime}} \psi(\sum_{j \in A} s_{ij}) + \sum_{i \in A} \psi(\sqrt{n})
\end{align}
This holds because the $S$ kernel is an identity kernel within $V$ and $V^{\prime}$ and only has terms in the cross between the two sets. Similarly, \begin{align}
    f(Q) =  \sum_{i \in V} \psi(\sum_{j \in Q} s_{ij}) + \eta \sum_{i \in Q} \psi(\sqrt{n})
\end{align}
Finally, we obtain $f(A \cup Q)$:
\begin{align}
    f(A \cup Q) =  \eta \sum_{i \in V^{\prime} \setminus Q} \psi(\sum_{j \in A} s_{ij}) + \eta \sum_{i \in Q} \psi(\sqrt{n}) + \sum_{i \in V \setminus A} \psi(\sum_{j \in Q} s_{ij}) + \sum_{i \in A} \psi(\sqrt{n})
\end{align}
Combining all the three terms together, we obtain $f(A) + f(Q) - f(A \cup Q) = \eta \sum_{i \in A} \psi(\sum_{j \in Q}s_{ij}) + \sum_{j \in Q} \psi(\sum_{i \in A} s_{ij}) = F_{\eta}(A; Q)$.

Finally, to show that $f_{\eta}(S)$ is restricted submodular, notice that $f(A)$ is submodular if $A$ is restricted to either $V$ or $V^{\prime}$. Similarly, given sets $A \subseteq V, B \subseteq V^{\prime}$, it holds that $f(A) + f(B) - f(A \cup B) = I_f(A; B) \geq 0$ which implies the restricted submodularity of $f_{\eta}(S)$.
\end{proof}

\section{Sample computation of gradients}

\subsection{Generic Summarization}

For the purpose of learning the parameters $w_i$ and $\lambda_i$, we compute the gradients as, $$\frac{\partial L_n}{\partial w_i} = f_i(\hat{Y_n}, x^{(n)}, \lambda_i) - f_i(Y^{(n)}, x^{(n)}, \lambda_i)$$ and $$\frac{\partial L_n}{\partial \lambda_i} = w_i \frac{\partial f_i(\hat{Y_n}, x^{(n)}, \lambda_i)}{\partial \lambda_i}  - w_i \frac{\partial f_i(Y^{(n)}, x^{(n)}, \lambda_i)}{\partial \lambda_i} $$ where $$\hat{Y_n} = \argmax_{Y \subset V^{(n)}, |Y| \leq k} F(Y, x^{(n)}, w, \lambda) + l_n(Y) $$

For the gradients with respect to the respective internal parameters of individual function components $\frac{\partial f_i}{\partial \lambda_i}$, consider the generalized graphcut  $f(Y, x^{(n)}, \lambda) =   \sum_{i \in V, j \in Y} s_{ij}^{(n)} - \lambda \sum_{i, j \in Y} s_{ij}^{(n)}$ as an example. We compute its gradient as,\\
$$\frac{\partial f (Y, x^{(n)}, \lambda)}{\partial \lambda} = -\sum_{i, j \in Y} s_{ij}^{(n)} $$

\subsection{Query-Focused Summarization}
For the purpose of learning the parameters $w_i$, $\lambda_i$ and $\eta_i$ we compute the gradients as, $$\frac{\partial L_n}{\partial w_i} = I_{f_i}(\hat{Y_n}, Q^{(n)}, x^{(n)}, \lambda_i, \eta_i) - I_{f_i}(Y^{(n)}, Q^{(n)}, x^{(n)}, \lambda_i, \eta_i)$$ $$\frac{\partial L_n}{\partial \lambda_i} = w_i \frac{\partial I_{f_i}(\hat{Y_n}, Q^{(n)}, x^{(n)}, \lambda_i, \eta_i)}{\partial \lambda_i}  - w_i \frac{\partial I_{f_i}(Y^{(n)}, Q^{(n)}, x^{(n)}, \lambda_i, \eta_i)}{\partial \lambda_i} $$ and $$\frac{\partial L_n}{\partial \eta_i} = w_i \frac{\partial I_{f_i}(\hat{Y_n}, Q^{(n)}, x^{(n)}, \lambda_i, \eta_i)}{\partial \eta_i}  - w_i \frac{\partial I_{f_i}(Y^{(n)}, Q^{(n)}, x^{(n)}, \lambda_i, \eta_i)}{\partial \eta_i} $$ where $$\hat{Y_n} = \argmax_{Y \subset V^{(n)}, |Y| \leq k} F(Y, Q^{(n)}, x^{(n)}, w, \lambda, \eta) + l_n(Y) $$\\

For the gradients of individual function components $\frac{\partial I_{f_i}}{\partial \eta_i}$ with respect to the respective query-relevance-diversity trade-off parameters $\eta_i$, we show computation for some functions as follows:

\textbf{FL1MI:} $I_f(Y, Q^{(n)}, x^{(n)}, \eta)=\sum_{i \in V}\min(\max_{j \in Y}s_{ij}^{(n)}, \eta \max_{j \in Q^{(n)}}s_{ij}^{(n)})$\\
    $$ \frac{\partial I_f (Y, Q^{(n)}, x^{(n)}, \eta)}{\partial \eta} = \sum_{i \in V} (\max_{j \in Q^{(n)}}s_{ij}^{(n)}*1_{\max_{j \in Q^{(n)}}s_{ij}^{(n)} \leq \max_{j \in Y}s_{ij}^{(n)}})$$
    
    \textbf{FL2MI:} $I_f(Y, Q^{(n)}, x^{(n)}, \eta) = \sum_{i \in Q^{(n)}} \max_{j \in Y} s_{ij}^{(n)} + \eta \sum_{i \in Y} \max_{j \in Q^{(n)}} s_{ij}^{(n)}$\\
    $$ \frac{\partial I_f (Y, Q^{(n)}, x^{(n)}, \eta)}{\partial \eta} = \sum_{i \in Y} \max_{j \in Q^{(n)}} s_{ij}^{(n)}$$
    
    \textbf{LogDetMI:} $I_f(Y, Q^{(n)}, x^{(n)}, \eta) = -\log \det(I - \eta^2 S_{Y}^{-1}S_{YQ^{(n)}}S_{Q^{(n)}}^{-1}S_{YQ^{(n)}}^T)$
    
    We have, 
    $$\frac{-\partial\log\det(X)}{\partial \eta} = \frac{-1}{\det(X)} \frac{\det (X)}{\partial \eta} $$ and $\frac{\partial \det (X)}{\partial \eta} = \det(X) \mathrm{Tr}[ X^{-1} \frac{\partial X}{\partial \eta}]$. Hence, with $X = I - \eta^2 S_{Y}^{-1}S_{YQ^{(n)}}S_{Q^{(n)}}^{-1}S_{YQ^{(n)}}^T$ we have,
    $$ \frac{\partial I_f (Y, Q^{(n)}, x^{(n)}, \eta)}{\partial \eta} = \mathrm{Tr}[((I - \eta^2 S_{Y}^{-1}S_{YQ^{(n)}}S_{Q^{(n)}}^{-1}S_{YQ^{(n)}}^T))^{-1}*2\eta(S_{Y}^{-1}S_{YQ^{(n)}}S_{Q^{(n)}}^{-1}S_{YQ^{(n)}}^T)]$$

\subsection{Privacy Preserving Summarization}

For the purpose of learning the parameters $w_i$, $\lambda_i$ and $\nu_i$ we compute the gradients as, $$\frac{\partial L_n}{\partial w_i} = f_i(\hat{Y_n}, P^{(n)}, x^{(n)}, \lambda_i, \nu_i) - f_i(Y^{(n)}, P^{(n)}, x^{(n)}, \lambda_i, \nu_i)$$ $$\frac{\partial L_n}{\partial \lambda_i} = w_i \frac{\partial f_i(\hat{Y_n}, P^{(n)}, x^{(n)}, \lambda_i, \nu_i)}{\partial \lambda_i}  - w_i \frac{\partial f_i(Y^{(n)}, P^{(n)}, x^{(n)}, \lambda_i, \nu_i)}{\partial \lambda_i} $$ and $$\frac{\partial L_n}{\partial \nu_i} = w_i \frac{\partial f_i(\hat{Y_n}, P^{(n)}, x^{(n)}, \lambda_i, \nu_i)}{\partial \nu_i}  - w_i \frac{\partial f_i(Y^{(n)}, P^{(n)}, x^{(n)}, \lambda_i, \nu_i)}{\partial \nu_i} $$ where $$\hat{Y_n} = \argmax_{Y \subset V^{(n)}, |Y| \leq k} F(Y, P^{(n)}, x^{(n)}, w, \lambda, \nu) + l_n(Y) $$

For the gradients of individual function components $\frac{\partial f_i}{\partial \nu_i}$ with respect to the respective privacy sensitivity parameters $\nu_i$, we show computation for some functions as follows:

    \noindent \textbf{FLCondGain: } $f(Y, P^{(n)}, x^{(n)}, \nu)= \sum_{i \in V} \max(\max_{j \in Y} s_{ij}^{(n)} - \nu \max_{j \in P^{(n)}} s_{ij}^{(n)}, 0)$\\
    $$ \frac{\partial f(Y, P^{(n)}, x^{(n)}, \nu)}{\partial \nu}  = \sum_{i \in V} (-\max_{j \in P^{(n)}}s_{ij}^{(n)})*1_{(\max_{j \in Y} s_{ij}^{(n)} - \nu \max_{j \in P^{(n)}} s_{ij}^{(n)}) \geq 0}$$
    \textbf{LogDetCondGain: } $f(Y, P^{(n)}, x^{(n)}, \nu) = \log\det(S_Y - \nu^2 S_{YP^{(n)}}S_{P^{(n)}}^{-1}S_{YP^{(n)}}^T)$
    We have, 
    $$\frac{\partial\log\det(X)}{\partial \nu} = \frac{1}{\det(X)} \frac{\det (X)}{\partial \nu} $$ and $\frac{\partial \det (X)}{\partial \nu} = \det(X) \mathrm{Tr}[ X^{-1} \frac{\partial X}{\partial \nu}]$
    Hence, with $X = I - \nu^2 S_{Y}^{-1}S_{YP^{(n)}}S_{P^{(n)}}^{-1}S_{YP^{(n)}}^T$\\
    we have, 
    $$\frac{\partial f(Y, P^{(n)}, x^{(n)}, \nu)}{\partial \nu}  = -\mathrm{Tr}[(S_Y - \nu^2 S_{YP^{(n)}}S_{P^{(n)}}^{-1}S_{YP^{(n)}}^T)^{-1}*2\nu(S_{YP^{(n)}}S_{P^{(n)}}^{-1}S_{YP^{(n)}}^T))] $$
    
\subsection{Joint Summarization}

For the purpose of learning the parameters in $\Theta$ we compute the gradients as, $$\frac{\partial L_n}{\partial w_i} = f_i(\hat{Y_n}, Q^{(n)}, P^{(n)}, x^{(n)}, \lambda_i, \eta_i, \nu_i) - f_i(Y^{(n)}, Q^{(n)}, P^{(n)}, x^{(n)}, \lambda_i, \eta_i, \nu_i)$$ $$\frac{\partial L_n}{\partial \lambda_i} = w_i \frac{\partial f_i(\hat{Y_n}, Q^{(n)}, P^{(n)}, x^{(n)}, \lambda_i, \eta_i, \nu_i)}{\partial \lambda_i}  - w_i \frac{\partial f_i(Y^{(n)}, Q^{(n)}, P^{(n)}, x^{(n)}, \lambda_i, \eta_i, \nu_i)}{\partial \lambda_i} $$ $$\frac{\partial L_n}{\partial \eta_i} = w_i \frac{\partial f_i(\hat{Y_n}, Q^{(n)}, P^{(n)}, x^{(n)}, \lambda_i, \eta_i, \nu_i)}{\partial \eta_i}  - w_i \frac{\partial f_i(Y^{(n)}, Q^{(n)}, P^{(n)}, x^{(n)}, \lambda_i, \eta_i, \nu_i)}{\partial \eta_i} $$ $$\frac{\partial L_n}{\partial \nu_i} = w_i \frac{\partial f_i(\hat{Y_n}, Q^{(n)}, P^{(n)}, x^{(n)}, \lambda_i, \eta_i, \nu_i)}{\partial \nu_i}  - w_i \frac{\partial f_i(Y^{(n)}, Q^{(n)}, P^{(n)}, x^{(n)}, \lambda_i, \eta_i, \nu_i)}{\partial \nu_i} $$ where $$\hat{Y_n} = \argmax_{Y \subset V^{(n)}, |Y| \leq k} F(Y, Q^{(n)}, P^{(n)}, x^{(n)}, w, \lambda, \eta, \nu) + l_n(Y) $$

\section{Behavior of Different Functions in Case of Query-Focused Summarization on Synthetic Data (Single-Query Case)}

Since our framework is rich enough to support multiple queries and/or multiple elements in private set, we verify for the consistency in the behavior of functions when applied to a single query case and when applied to multiple query case. In the main paper we report the results of analysis on multiple query case. Here we report the results when single query point is used to produce query-focused summaries. As expected, we observe similar behavior of functions in single query case as well. For example, in Figure~\ref{fig:query-single-fl1}(Left), similar to the multiple query case, FL1MI reaches saturation after picking up a query matching point and then picks up points arbitrarily. In Figure~\ref{fig:query-single-fl1}(Middle) we see that by combining FL1MI with a diversity function, say DisparitySum, with even a very small weight, the elements picked up after saturation are diverse. We also observe similar effect of $\eta$ in governing the query-relevance and saturation trade-off here. As $\eta$ is increased summary produced by FL1MI increasingly becomes more query-relevant (Figure~\ref{fig:query-single-fl1}(Right)). This is in contrast with the behavior of the alternative formulation FL2MI which in the similar way starts off with saturation (first image in Figure~\ref{fig:query-single-others}) but even a slight positive non-zero $\eta$ immediately makes the summary highly query relevant (second image in Figure~\ref{fig:query-single-others}). The third and fourth images in Figure~\ref{fig:query-single-others} similarly show the behavior of COM and GCMI in selecting highly query-relevant summaries. The behavior of LogDetMI is also worth noting. As seen in Figure~\ref{fig:query-single-logdet}, with $\eta=1$ LogDetMI selects a highly query-relevant summary with no saturation. However as $\eta$ is increased, the summary starts becoming more diverse and less query-relevant. While there is indeed some effect of $\eta$, the trend in query-relevance is neither clearly increasing nor clearly decreasing.  

\begin{figure*}[h]
    \includegraphics[width=\textwidth]{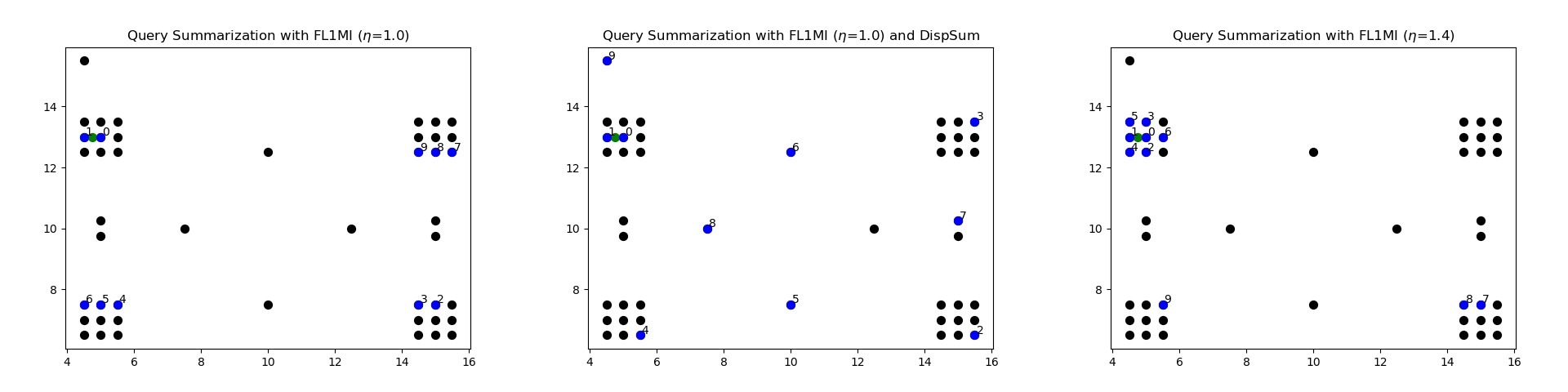}
    \caption{Left to Right: FL1MI exhibiting saturation; helped by slight additional diversity; increasing $\eta$ delays saturation }
    \label{fig:query-single-fl1}
\end{figure*}

\begin{figure*}[h]
    \includegraphics[width=\textwidth]{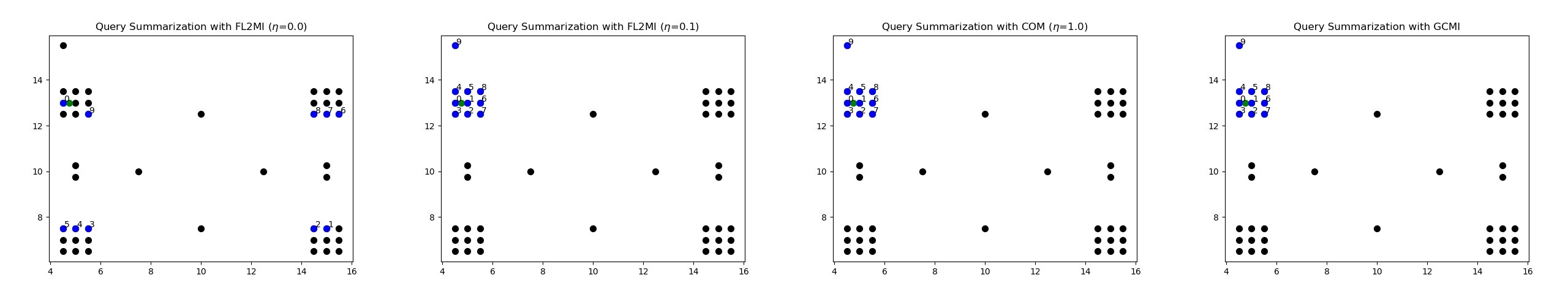}
    \caption{Behavior of FL2MI, COM and GCMI in single query case}
    \label{fig:query-single-others}
\end{figure*}

\begin{figure*}[h]
    \includegraphics[width=\textwidth]{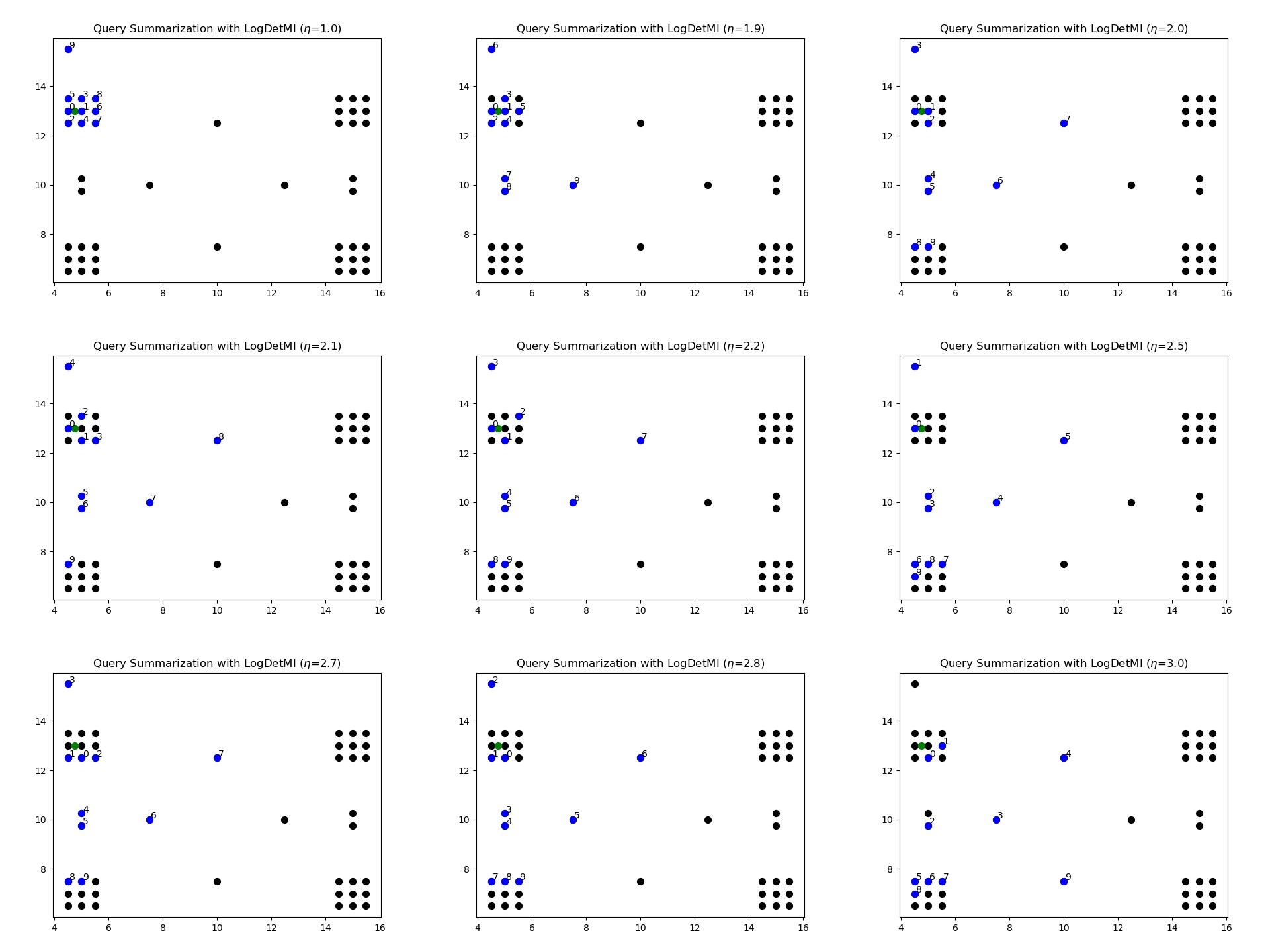}
    \caption{Role of $\eta$ in LogDetMI for single query case}
    \label{fig:query-single-logdet}
\end{figure*}

\end{document}